\documentclass[journal]{IEEEtran}
\usepackage{xcolor,colortbl,bbm}
\usepackage{algorithmic,algorithm}
\usepackage{graphicx, subcaption}
\usepackage{balance}
\usepackage{amssymb}

\usepackage{cite}

\ifCLASSINFOpdf
\else
\fi
\usepackage{amssymb, amsmath, amsthm,bm,dsfont, float}
\DeclareFontFamily{OT1}{pzc}{}
\DeclareFontShape{OT1}{pzc}{m}{it}{<-> s * [1.10] pzcmi7t}{}
\DeclareMathAlphabet{\mathpzc}{OT1}{pzc}{m}{it}

\newcommand{\col}[1]{\mbox{\rm col}\left\{#1\right\}}
\newcommand{\grad}[1]{\nabla_{#1^{\tran}} }
\newcommand{\ws}{\bm{{\scriptstyle\mathcal{W}}}}
\newcommand{\w}{\bm{w}}

\newcommand{\we}{\widetilde{\w}}
\newcommand{\eqdef}{\:\overset{\Delta}{=}\:}
\DeclareMathOperator*{\argmin}{argmin}
\newcommand{\Li}[1]{\mathcal{L}_{#1,i}}

\newcommand{\tran}{{\sf T}}

\newtheorem{theorem}{Theorem}
\newtheorem{assumption}{Assumption}
\newtheorem{lemma}{Lemma}
\newtheorem{corollary}{Corollary}

\newtheorem{definition}{Definition}

\definecolor{Gray}{gray}{0.8}
\definecolor{LightCyan}{rgb}{0.88,1,1}

\makeatletter
\def\endthebibliography{%
	\def\@noitemerr{\@latex@warning{Empty `thebibliography' environment}}%
	\endlist
}
\makeatother

\begin{document}
%
\title{Privatized Graph Federated Learning}
%
%
%

\author{Elsa~Rizk\IEEEauthorrefmark{1},~\IEEEmembership{Member,~IEEE,}
        Stefan~Vlaski\IEEEauthorrefmark{2},~\IEEEmembership{Member,~IEEE,}
        and~Ali~H. Sayed\IEEEauthorrefmark{1},~\IEEEmembership{Fellow,~IEEE.}
\thanks{\IEEEauthorrefmark{1}The authors are with the School of Engineering, École Polytechnique Fédérale de Lausanne (e-mail: \{elsa.rizk; ali.sayed\}@epfl.ch).\IEEEauthorrefmark{2}The author is with the Department of Electrical and Electronic Engineering, Imperial College London (e-mail: s.vlaski@imperial.ac.uk).
A short conference article dealing with an earlier version of this work without extended arguments and proofs appears in \cite{GFL}. 
\\ This work has been submitted for review.}
}

\maketitle

\vspace*{-0.5cm}
\begin{abstract}
Federated learning is a semi-distributed algorithm, where a server communicates with multiple dispersed clients to learn a global model. The federated architecture is not robust and is sensitive to communication and computational overloads due to its one-master multi-client structure. It can also be subject to privacy attacks targeting personal information on the communication links. In this work, we introduce graph federated learning (GFL), which consists of multiple federated units connected by a graph. We then show how graph homomorphic perturbations can be used to ensure the algorithm is differentially private. We conduct both convergence and privacy theoretical analyses and illustrate performance by means of computer simulations. 
\end{abstract}


%
\IEEEpeerreviewmaketitle

\vspace{-0.5cm}
\section{Introduction}

\IEEEPARstart{F}{e}derated learning (FL) \cite{mcmahan16} emerged in recent years as a solution to distributed machine learning where  users no longer need to send their data to a server for training. Instead, data remains local, and training happens in collaboration between different clients and the server. Compared to a fully decentralized solution, communication occurs between the server and the clients (or agents), instead of directly between the agents themselves. Such a solution is advantageous in the sense that users no longer need to worry about sharing their data with an unknown party, and the high cost of sending all their raw data is eliminated. In this way, the data stays locally safe on a user's device, and no extra communication cost is incurred for transferring the data remotely. However, such a distributed architecture is not robust to communication failures and computational overloads, nor it is immune to privacy attacks when agents are required to share their local updates. In standard FL, millions of users can be connected to \textit{one} server at a time. This means one server will need to be responsible for the communication with all clients with significant computational burden, thus rendering the system susceptible to communication failures. Furthermore, whether clients send their gradient updates or their local models, information about their data can be inferred from the exchanges and leaked \cite{Hitaj2017,Melis2019,nasr2019comprehensive,Zhu2019}. Consider for instance the logistic risk; the gradient of the loss function is a constant multiple of the feature vector. Thus, even though the actual data samples are not sent to the server, information about them can still be inferred from the gradient updates or the models.  

These considerations motivate us to propose an architecture for federated learning with privacy guarantees. In particular, we introduce the graph federated architecture, which consists of multiple servers, and we privatize the algorithm by ensuring the communication ocuring between the servers and the clients is secure. Graph homomorphic perturbations, which were initially introduced in \cite{vlaski2020graphhomomorphic}, focus on the communication between servers. They are based on adding correlated noise to the messages sent between servers such that the noise cancels out if we were to take the average of all messages across all servers. As for the privatization between the clients and their servers, we rely on already existing solutions, mainly \cite{bonawitz2016practical}. It is a secret sharing method that guarantees messages sent by clients are encoded and is equivalent to masking the message. While the two protocols rely on different tools, both make sure the effect of the added noise is reduced.

Other works have also contributed to addressing the same challenges we are considering in this work, albeit differently. For example, the work \cite{HFL} introduces a hierarchical architecture, where it is assumed there are multiple servers connected in a tree structure. Such a solution still has one main server and thus faces the same robustness problem as FL. The graph federated learning architecture in this work (and which appeared in the earlier conference publication \cite{GFL}) is a more general structure. The work \cite{liu2022decentralized} generalizes the standard distributed learning framework to include local updates, while \cite{wang2022confederated} has a similar architecture to the GFL architecture proposed earlier in \cite{GFL}, it nevertheless does not deal with privacy and employs different objective functions and a different learning algorithm based on the alternating direction method of multipliers. Likewise, a plethora of solutions exist that relate to privacy issues. These methods may be split into two sub-groups: those using random perturbations to ensure a certain level of differential privacy \cite{geyer2017differentially,hu2020personalized,triastcyn2019federated,truex2020ldp,wei2020federated,JayDLDP,LiDLDP,ZhuDPDL,pathak2010DP}, or those that rely on cryptographic methods \cite{bonawitz2016practical,gascon2017privacy,Mohassel2017SecureML,Niko2013,Zheng2019}. 
Both have their advantages and disadvantages. While differential privacy is easy to implement, it hinders the performance of the algorithm by reducing the model utility. As for cryptographic methods, they are generally harder to implement since they require more computational and communication power \cite{ishai2008cryptography,damgaard2010perfectly}. Furthremore, they restrict the number of participating users. Thus, moving forward, we believe a scheme fusing the two paradigms would be best. 

\vspace*{-0.3cm}
\section{Graph Federated Architecture}
In the graph federated architecture, which we initially introduced in \cite{GFL}, we consider $P$ federated units connected by a graph structure. Each federated unit consists of a server and a set of $K$ agents. Thus, the overall architecture can be represented as a graph depicted in Figure \ref{fig:GFA}. We denote the combination matrix connecting the servers by $A \in \mathbb{R}^{P\times P }$, and we write $a_{mp}$ to refer to the elements of $A$. We assume each agent of every server has its own dataset $\{x_{p,k,n}\}_{n=1}^{N_{p,k}}$ that is non-iid when compared to the other agents. The subscript $p$ refers to the federated unit, $k$ to the agent, and $n$ to the data sample. 

\begin{figure}
	\centering
		\includegraphics[width = 0.47\textwidth]{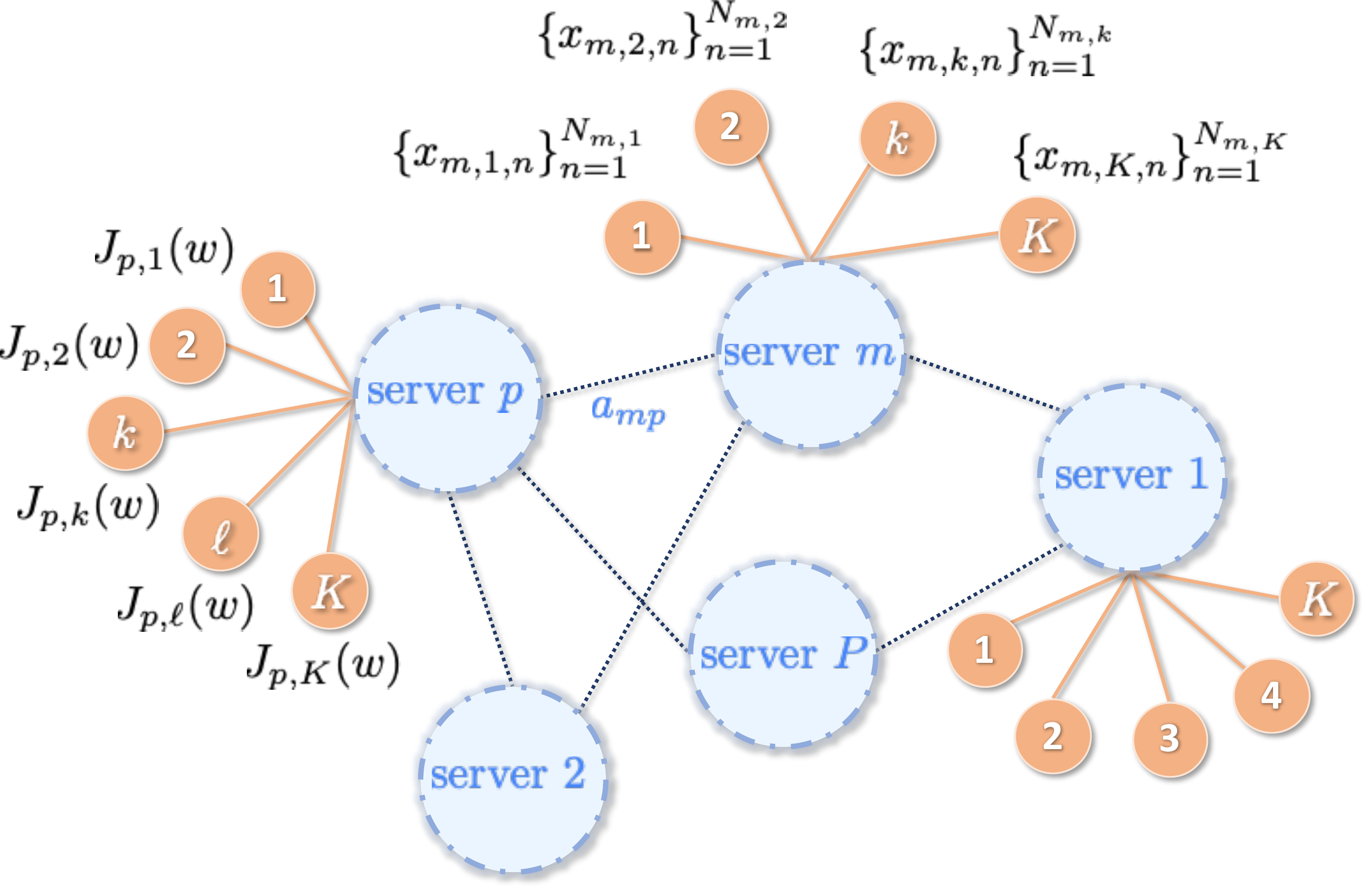}
		\caption{The graph federated learning architecture.}\label{fig:GFA}
\end{figure}

With this architecture, we associate a convex optimization problem that will take into account the cost function at each federated unit. Thus, the optimization goal is to find the optimal global model $w^o$ that minimizes an average empirical risk:
\begin{equation}\label{eq:optProb}
	w^o \eqdef \argmin_{w\in \mathbb{R}^M} \frac{1}{P}\sum_{p=1}^P \frac{1}{K}\sum_{k=1}^K J_{p,k}(w),
\end{equation}
where each individual cost is an empirical risk defined over the local loss functions $Q_{p,k}(\cdot;\cdot)$:
\begin{equation}\label{eq:locER}
	J_{p,k} (w) \eqdef \frac{1}{N_{p,k}} \sum_{n=1}^{N_{p,k}} Q_{p,k}(w;x_{p,k,n}).
\end{equation}
To solve problem \eqref{eq:optProb} each federated unit $p$ runs the standard federated averaging (FedAvg) algorithm \cite{mcmahan16}. An iteration $i$ of the algorithm consists of the server $p$ selecting a subset of $L$ participating agents $\Li{p}$. Then, in parallel, each agent runs a series of stochastic gardient descent (SGD) steps. We call these local steps epochs, and denote an epoch by the letter $e$ and the total number of epochs by $E_{p,k}$. The sampled data point at an agent $k$ in the federated unit $p$ during the $e^{th}$ epoch of iteration $i$ is denoted by $b$. Thus, during an iteration $i$, each participating agent $k \in \Li{p}$ updates the last model $\w_{p,i-1}$ and sends its new model $\w_{p,k,E_{p,k}}$ to the server after $E_{p,k}$ epochs. During a single epoch $e$, the agent updates its current local model $w_{p,k,e-1}$ by running a single SGD step. Thus, an agent repeats the following adaptation step for $e=1,2,\cdots, E_{p,k}$: 
\begin{align}\label{eq:UpdStep}
	\w_{p,k,e} = &\:\w_{p,k,e-1} - \frac{ \mu }{E_{p,k}}  \grad{w}Q_{p,k}(\w_{p,k,e-1};\bm{x}_{p,k,b}),
\end{align}
with $\bm{x}_{p,k,b} $ be the sampled data of agent $k$ in federated unit $p$, and $\w_{p,k,0} = \w_{p,i-1}$. After all the participating agents $k \in \Li{p}$ run all their epochs, the server aggregates their final models $\w_{p,k,E_{p,k}}$, which we rename as $\w_{p,k,i}$ since it is the final local model at iteration $i$:
\begin{align}\label{eq:aggrStep}
	\bm{\psi}_{p,i} = \frac{1}{L}\sum_{k\in \Li{p}} \w_{p,k,i}.
\end{align}
Next, at the server level, these estimates are combined across neighbourhoods using a diffusion type strategy, where we first consider the previous steps \eqref{eq:UpdStep} and \eqref{eq:aggrStep} as the adaptation step and the following step as the combination step:
\begin{equation}\label{eq:combStep}
	\w_{p,i} = \sum_{m\in \mathcal{N}_p}a_{pm}\bm{\psi}_{m,i}.
\end{equation}

To introduce privacy, the models communicated at each round between the agents and the servers need to be encrypted in some way. We could either apply secure multiparty computation (SMC) tools, like secret sharing, or use differential privacy. We focus on differential privacy or masking tools that can be represented by added noise. Thus, we let agent 1 in federated unit 2 add a noise component $\bm{g}_{2,1,i}$ to its final model $\w_{2,1,i}$ at iteration $i$, and then let serever $2$ add $\bm{g}_{12,i}$ to the message $\bm{\psi}_{2,i}$ it sends to server 1. More generally, we denote by $\bm{g}_{pm,i}$ the noise added to the message sent by server $m$ to server $p$ at iteration $i$. Similarly, we denote by $\bm{g}_{p,k,i}$ the noise added to the model sent by agent $k$ to server $p$ during the $i^{th}$ iteration. We use unseparated subscripts $pm$ for the inter-server noise components to point out their ability to be combined into a matrix structure. Contrarily, the agent-server noise components' subscripts are separated by a comma to highlight a hierarchical structure. Thus, the privatized algorithm can be written as a client update step \eqref{eq:privUpStep}, a server aggregation step \eqref{eq:privAggrStep}, and a server combination step \eqref{eq:privCombStep}: \vspace*{-0.5cm}
\begin{align}
	\w_{p,k,i} &= \w_{p,i-1} - \frac{\mu}{E_{p,k}} \sum_{e=1}^{E_{p,k}}  \grad{w}Q_{p,k}(\w_{p,k,e-1};\bm{x}_{p,k,b}), \label{eq:privUpStep} 
	\\
	\bm{\psi}_{p,i} &=   \frac{1}{L}\sum_{k \in \Li{p}} \w_{p,k,i} + \bm{g}_{p,k,i}, \label{eq:privAggrStep} 
	\\
	\w_{p,i} &= \sum_{m\in \mathcal{N}_p} a_{pm}(\bm{\psi}_{m,i} + \bm{g}_{pm,i}). \label{eq:privCombStep}
\end{align}
The client update step \eqref{eq:privUpStep} follows from \eqref{eq:UpdStep} by combining the multiple epochs for $e=1,2,\cdots, E_{p,k}$ into one update step, with $\w_{p,k,i} = \w_{p,k,E_{p,k}}$ and $\w_{p,k,0} = \w_{p,i-1}$, namely:
\begin{align}
	&\w_{p,k,E_{p,k}} \notag \\
	&= \w_{p,k,E_{p,k}-1}  
	 - \frac{\mu }{E_{p,k}} \grad{w}Q_{p,k}(\w_{p,k,E_{p,k}-1};\bm{x}_{p,k,b}) 
	\notag \\
	&= \w_{p,k,E_{p,k,}-2} 
 -\frac{\mu}{E_{p,k}} \sum_{e=E_{p,k}-1}^{E_{p,k}}   \grad{w}Q_{p,k}(\w_{p,k,e-1};\bm{x}_{p,k,b}) 
	\notag \\
	& = \w_{p,k,0} 
 -\frac{\mu}{E_{p,k}}\sum_{e=1}^{E_{p,k}}
	\grad{w}Q_{p,k}(\w_{p,k,e-1};\bm{x}_{p,k,b}).
\end{align}

\vspace{-0.6cm}
\section{Performance Analysis}
In this section, we show a list of results on the performance of the algorithm. We study the convergence of the privatized algorithm \eqref{eq:privUpStep}-\eqref{eq:privCombStep}, and examine the effect of privatization on performance. 

\subsection{Modeling Conditions}
To go forward with our analysis, we require certain reasonable assumptions on the graph structure and cost functions. 

\begin{assumption}[\textbf{Combination matrix}]\label{ass:adj}
	The combination matrix $A$ describing the graph is symmetric and doubly-stochastic, i.e.: \vspace*{-0.5cm}
	\begin{equation}\label{eq:assAdjSto}
		a_{pm} = a_{mp}, \quad \sum_{m=1}^P a_{mp} = 1.
	\end{equation}
	Furthermore, the graph is strongly-connected and $A$ satisfies:
	\begin{equation}\label{eq:assAdjFull}
	\iota_2 \eqdef \rho\left(A -\frac{1}{P}\mathbbm{1}\mathbbm{1}^\tran \right ) < 1.
	\end{equation}
\qed
\end{assumption}

\begin{assumption}[\textbf{Convexity and smoothness}]\label{ass:fct}
	The empirical risks $J_{p,k}(\cdot)$ are $\nu-strongly$ convex with $\nu > 0$, and the loss functions $Q_{p,k}(\cdot;\cdot)$ are convex, namely:
	\begin{align} \label{eq:assFctConv}
		J_{p,k}(w_2) \geq &J_{p,k}(w_1) + \grad{w}J_{p,k}(w_1)(w_2-w_1) \notag \\
		& + \frac{\nu}{2}\Vert w_2 - w_1 \Vert^2, \\
		Q_{p,k}(w_2;\cdot) \geq & Q_{p,k}(w_1;\cdot) + \grad{w}Q_{p,k}(w_1;\cdot) (w_2 - w_1).
	\end{align}
	Furthermore, the loss functions have $\delta-$Lipschitz continuous gradients:
	\begin{equation}\label{eq:assFctLip}
		\Vert \grad{w}Q_{p,k}(w_2;x_{p,k,n}) - \grad{w}Q_{p,k}(w_1;x_{p,k,n})\Vert \leq \delta \Vert w_2 - w_1\Vert.
	\end{equation}
\qed
\end{assumption}

We also require a bound on the difference between the global optimal model $w^o$ and the local optimal models $w^o_{p,k}$ that optimize $J_{p,k}(\cdot)$. This assumption is used to bound the gradient noise and the incremental noise defined further ahead. It is not a restrictive assumption, and it imposes a condition on when collaboration is sensical among different agents. In other words, since the agents have non-iid data, sometimes their optimal models are too different and collaboration would hurt their individual performance. For example, when considering recommender systems, people in the same country are more likely to get the same movie recommended as opposed to accross different countries. This means, people of the same country might have different models but relatively close contrary to different countries. 

\begin{assumption}[\textbf{Model drifts}]\label{ass:mod}
	The distance of each local model $w_{p,k}^o$ to the global model $w^o$ is uniformly bounded, $\Vert w^o - w_{p,k}^o\Vert \leq \xi$. 
\end{assumption}

\vspace*{-0.5cm}
\subsection{Network Centroid Convergence}
We study the convergence of the algorithm from the network centroid's $\w_{c,i}$ perspective: 
\begin{equation}\label{eq:netCent}
	\w_{c,i} \eqdef \frac{1}{P}\sum_{p=1}^P \w_{p,i}.
\end{equation}
We write the central recursion as:

\begin{align}\label{eq:centRec}
	\w_{c,i} = &\: \w_{c,i-1} - \mu \frac{1}{PL}\sum_{p=1}^P   \sum_{k\in \Li{p}} \frac{1}{E_{p,k}} \sum_{e=1}^{E_{p,k}}\notag \\
	& \times  \grad{w}Q_{p,k} (\w_{p,k,e-1};\bm{x}_{p,k,b}) \notag \\
	&+ \frac{1}{PL} \sum_{p=1}^P \sum_{k\in\Li{p}} \bm{g}_{p,k,i} + \frac{1}{P}\sum_{p,m = 1}^P a_{pm} \bm{g}_{pm,i}.
\end{align}
Next, we define the model error as $\we_{c,i} \eqdef w^o - \w_{c,i}$ and the average gradient noise: \vspace*{-0.1cm}
\begin{equation}\label{eq:gradNoise}
	\bm{s}_i \eqdef \frac{1}{P}\sum_{p=1}^P \bm{s}_{p,i},
\end{equation}
with the per-unit gradient noise $\bm{s}_{p,i}$:
\begin{equation}\label{eq:locgradNoise}
	\bm{s}_{p,i} \eqdef \widehat{\grad{w}J_{p}}(\w_{p,i-1}) - \grad{w}J_p(\w_{p,i-1}),
\end{equation}
and:
\begin{align}\label{eq:servStocGrad}
	\widehat{\grad{w}J_p}(\cdot) &\eqdef \frac{1}{L} \sum_{k\in\Li{p}} \frac{1}{E_{p,k}} \sum_{e=1}^{E_{p,k}} \grad{w}Q_{p,k}(\cdot;\bm{x}_{p,k,b}).
\end{align}
We introduce the average incremental noise $\bm{q}_i$ and the local incremental noise $\bm{q}_{p,i}$, which capture the error introduced by the multiple local update steps:
\begin{align}
	\bm{q}_i &\eqdef \frac{1}{P}\sum_{p=1}^P \bm{q}_{p,i}, \label{eq:incNoise} \\
	\bm{q}_{p,i} &\eqdef \frac{1}{L} \sum_{k \in \Li{p}} \frac{1}{E_{p,k}} \sum_{e=1}^{E_k}  \Big( \grad{w}Q_{p,k}(\w_{p,k,e-1}; \bm{x}_{p,k,b})
	 \notag \\
	&\quad  - \grad{w}Q(\w_{p,i-1}; \bm{x}_{p,k,b})\Big)
\end{align} 
We then arrive at the following error recursion: 
\begin{equation}\label{eq:centErrRec}
	\we_{c,i} = \we_{c,i-1} + \mu \frac{1}{P}\sum_{p=1}^P \grad{w}J_p(\w_{p,i-1}) + \mu \bm{s}_i + \mu \bm{q}_i- \bm{g}_{i},
\end{equation} 
where $\bm{g}_{i}$ is the total added noise at iteration $i$:
\begin{equation}\label{eq:centNoise}
	\bm{g}_{i} \eqdef \frac{1}{PL} \sum_{p=1}^P \sum_{k \in \Li{p}} \bm{g}_{p,k,i} + \frac{1}{P}\sum_{p,m=1}^P a_{pm}\bm{g}_{pm,i}
\end{equation}

We estimate the first and second-order moments of the gradient noise in the following lemma. To do so, we use the fact, shown in previous work (Lemma 1 in \cite{rizk2020federated}), that the individual gradient noise is zero-mean with a bounded second order moment: 
\begin{equation}\label{eq:bdSerGradNoise}
	\mathbb{E}\left\{ \Vert \bm{s}_{p,i} \Vert^2 | \mathcal{F}_{i-1}\right\} \leq \beta_{s,p}^2 \Vert \we_{p,i-1}\Vert^2 + \sigma_{s,p}^2,	
\end{equation}
where the constants are defined as:
\begin{align}
	\beta_{s,p}^2 &\eqdef \frac{6\delta^2}{L} \left( 1 + \frac{1}{K}\sum_{k=1}^K \frac{1}{E_{p,k}}\right) , \\
	\sigma_{s,p}^2 &\eqdef \frac{1}{LK}\sum_{k=1}^K   \left( \frac{12}{E_{p,k}} + 3\right) \notag \\
	&\times \frac{1}{N_{p,k}}\sum_{n=1}^{N_{p,k}} \Vert \grad{w}Q_{p,k}(w^o;x_{p,k,n})\Vert^2,
\end{align}
and $\mathcal{F}_{i-1}$ is the filtration defined over the randomness introduced by all the past subsampling of the data for the calculation of the stochastic gradient. 
Using Assumption \ref{ass:mod}, we can guarantee that $\sigma_{s,p}^2$ is bounded by bounding:
	\begin{equation}
		\Vert \grad{w}Q_{p,k}(w^o; x_{p,k,n})\Vert^2 \leq 2\Vert \grad{w}Q_{p,k}(w^o_{p,k};x_{p,k,n})\Vert^2 + 2\delta^2 \xi^2.
	\end{equation}

\begin{lemma}[\textbf{Estimation of first and second-order moments of the gradient noise}] \label{lem:gradNoise}
	The gradient noise defined in \eqref{eq:gradNoise} is zero-mean and has a bounded second-order moment:
	\begin{align}\label{eq:bdGradNoise}
		\mathbb{E}\left\{ \Vert \bm{s}_i \Vert^2 | \mathcal{F}_{i-1}  \right\}
		&\leq \beta_s^2 \Vert \we_{c,i-1}\Vert^2 + \sigma_s^2 \notag \\
		&\quad + \frac{2}{P}\sum_{p=1}^P \beta_{s,p}^2 \Vert \w_{p,i-1} - \w_{c,i-1}\Vert^2
	\end{align}
	where the constants $\beta_s^2$ and $\sigma_s^2$ are given by:
	\begin{align}
		\beta_s^2 &\eqdef \frac{2}{P}\sum_{p=1}^P \beta_{s,p}^2, \quad
		\sigma_s^2 \eqdef \frac{1}{P}\sum_{p=1}^P\sigma_{s,p}^2.
	\end{align}
\end{lemma}
\begin{proof}
	The above result follows from applying the Jensen's inequality and the bounds on the per-unit gradient noise $\bm{s}_{p,i}$.
\end{proof}

The new term found in the bound of the gradient term is what we call the network disagreement: 
\begin{equation}
	\frac{1}{P} \sum_{p=1}^P \Vert \w_{p,i} - \w_{c,i}\Vert^2.
\end{equation}
It captures the difference in the path taken by the individual models versus the network centroid. We bound this difference in Lemma \ref{lem:netDis}.   However, before doing so, we show that the second order moment of the incremental noise is on the order of $O(\mu)$. From Lemma 5 in \cite{rizk2020federated}, we can bound the individual incremental noise: 
\begin{align}
	\mathbb{E} \Vert \bm{q}_{p,i}\Vert^2 \leq  & a \mu^2 \mathbb{E} \Vert \we_{p,i-1}\Vert^2 + a \mu^2\xi^2 
	\notag \\
	&
	+ \frac{1}{K}\sum_{k=1}^K (b_k\mu^4 + c_k \mu^2)\sigma_{q,p,k}^2,
\end{align}
where the constants are given by:
\begin{align}
	a &\eqdef \frac{4\delta^2}{K}\sum_{k=1}^K  \frac{(E_{p,k}+1)(1-\lambda)-1+\lambda^{E_{p,k}+1}}{E_{p,k}^2(1-\lambda)^2}, \\
	b_k &\eqdef  \frac{2E_{p,k}(E_{p,k}+1)(1-\lambda)^2 - 4E_{p,k}(1-\lambda) + 4\lambda }{E_{p,k}^2 (1-\lambda)^3 } \notag \\
	&\quad -\frac{   2\lambda^{E_{p,k}+1}}{E_{p,k}^2 (1-\lambda)^3},\\
	c_k &\eqdef \frac{E_{p,k}-1}{3E_{p,k}}, \\
	\lambda &\eqdef 1-2\nu\mu + 4\delta^2\mu^2, \\
	\sigma^2_{q,p,k} & \eqdef 3\sum_{n=1}^{N_{p,k}} \Vert \grad{w}Q_{p,k}(w^o_{p,k};x_{p,k,n})\Vert^2.
\end{align}
The following result follows.
\begin{lemma}[\textbf{Estimation of second-order moment of the incremental noise}]
	The incremental noise defined in \eqref{eq:incNoise} has a bounded second-order moment:
	\begin{align}\label{eq:lemQBd}
		\mathbb{E} \Vert \bm{q}_i\Vert^2 \leq &O(\mu) \mathbb{E} \Vert \we_{c,i-1}\Vert^2 + O(\mu)\xi^2 + O(\mu^2 )\sigma_{q}^2  \notag \\
		& + \frac{O(\mu)}{P}\sum_{p=1}^P  \mathbb{E}\Vert \w_{p,i-1} - \w_{c,i-1}\Vert^2,
	\end{align}
where the constant $\sigma_q^2$ is the average of $\sigma_{q,p,k}^2$:
	\begin{equation}
		\sigma_{q}^2 \eqdef \frac{1}{PK}\sum_{p=1}^P\sum_{k=1}^K (b_k\mu^4 + c_k \mu^2)\sigma_{q,p,k}^2.
	\end{equation}
\end{lemma}
\begin{proof}
	The above result follows from applying the Jensen inequality and the bounds on the per-unit incremental noise $\bm{q}_{p,i}$. Furthermore, $a = O(\mu^{-1}), b_k = O(\mu^{-1}),$ and $c_k = O(1)$ reduce the expression to \eqref{eq:lemQBd}. 
\end{proof}

We now bound the network disagreement. To do so, we first introduce the eigendecomposition of $A = QH Q^\tran$:
\begin{equation}
	Q \eqdef \begin{bmatrix}
		\frac{1}{\sqrt{P}}\mathds{1} & Q_{\epsilon} 
	\end{bmatrix}, 
	\quad H \eqdef \begin{bmatrix}
		1 & 0 \\0 & H_{\epsilon}
	\end{bmatrix}, 
\end{equation}
where $H_{\epsilon}$ is a diagonal matrix that includes the last $(P-1)$ eigenvalues of $A$ and $Q_{\epsilon}$ their corresponding eigenvectors.

\begin{lemma}[\textbf{Network disagreement}]\label{lem:netDis}
	The average deviation from the centroid is bounded during each iteration $i$:
	\begin{align}
		&\frac{1}{P}\sum_{p=1}^P \mathbb{E}\Vert \w_{p,i} - \w_{c,i}\Vert^2 \notag \\
		&\leq  \frac{ \iota_2^i }{P} \mathbb{E} \Vert (Q_{\epsilon} \otimes I)\ws_0\Vert^2   +    \frac{\iota_2^2 }{P}\sum_{j'=0}^{i-1}\iota_2^{j'}\sum_{p=1}^P \Bigg( \mu^2\bigg(\frac{2\delta^2}{\iota_2(1-\iota_2) }   \notag \\
		&\quad +\beta_{s,p}^2 + O(\mu)  \bigg) \bigg( \lambda_p^{j'} A^{j'}[p] \col{\mathbb{E}\Vert \we_{p,0}\Vert^2}_{p=1}^P + \sum_{j=0}^{j'-1} \lambda_p^j  \notag \\
		&\quad  \times A^j[p]\col{\mu^2 \sigma_{s,p}^2 + O(\mu^2)\xi^2 + O(\mu^3)\sigma_{q,p}^2 + \sigma_{g,p}^2}_{p=1}^P \bigg) \notag \\
		&\quad + \mu^2\frac{2\Vert \grad{w}J_p(w^o)\Vert^2}{\iota_2(1-\iota_2)}+ \mu^2\sigma_{s,p}^2 + O(\mu^3)\xi^2 + O(\mu^4)\sigma_{q,p}^2 \notag \\
		&\quad+ \frac{1}{\iota_2^2} \sigma_{g,p}^2\Bigg),
	\end{align}
	where $\ws_{0} \eqdef \col{\w_{p,0}}_{p=1}^P$ and $\lambda_p \eqdef \sqrt{1-2\nu\mu + \delta^2\mu^2} + \beta_{s,p}^2 \mu^2 + O(\mu^2) \in (0,1)$.
	Then, in the limit:
	\begin{align}
		& \limsup_{i\to \infty} \frac{1}{P}\sum_{p=1}^P \mathbb{E} \Vert \w_{p,i} -\w_{c,i}\Vert^2 \notag \\
		&\leq     
		\frac{\iota_2^2}{P(1-\iota_2)} \sum_{p=1}^P \mu^2 \sigma_{s,p}^2 + \frac{1}{\iota_2^2}\sigma_{g,p}^2  + O(\mu)\sigma_{g,p}^2 + O(\mu^3).
	\end{align}
\end{lemma}
\begin{proof}
	Proof found in Appendix \ref{app:lemmNetDis}.
\end{proof}
Thus, from the above lemma, we see that the individual models gravitate to the centroid model with an error introduced due to the added privatization. The effect of the added noise overpowers that of the gradient and incremental noise, since the later is on the order of the step-size. 

Then, using the above result, we can establish the convergence of the centroid model to a neighbourhood of the true optimal model $w^o$. 

\begin{theorem}[\textbf{Centroid MSE convergence}]\label{thrm:centMSE}
	Under Assumptions \ref{ass:adj}, \ref{ass:fct} and \ref{ass:mod}, the network centroid converges to the optimal point $w^o$ exponentially fast for a sufficiently small step-size $\mu$:
	\begin{align}
		\mathbb{E} \Vert \we_{c,i}\Vert^2 \leq& \lambda_c \mathbb{E} \Vert \we_{c,i-1} \Vert^2  +\mu^2 \sigma_s^2 + O(\mu^{2})\xi^2 + O(\mu^{3})\sigma_q^2   \notag \\
		& + \mathbb{E} \Vert \bm{g}_{i}\Vert^2  + \frac{O(\mu)}{P}\sum_{p=1}^P \mathbb{E}\Vert \w_{p,i-1}-\w_{c,i-1}\Vert^2,
	\end{align}
	where $\lambda_c = \sqrt{1-2\nu\mu + \delta^2\mu^2}  +\beta_s^2\mu^2 + O(\mu^{2}) \in (0,1)$. Then, letting $i$ tend to infinity, we get:
	\begin{align}\label{eq:thrmBd}
		\limsup_{i\to \infty} \mathbb{E} \Vert \we_{c,i}\Vert^2 \leq & \frac{\mu^2 \sigma_s^2  + O(\mu^{2})\xi^2 + O(\mu^{3})\sigma_q^2 + \mathbb{E}\Vert \bm{g}\Vert^2}{1-\lambda_c}  \notag \\&  + \sum_{p=1}^PO(1 ) \sigma_{g,p}^2+ O(\mu).
	\end{align}
\end{theorem}
\begin{proof}
	Proof found in Appendix \ref{app:thrmCentMSE}.
\end{proof}
The main term in the above bound is the variance of the added noise with a  dominating factor of $\mu^{-1}$, since:
\begin{align}
	1- \lambda_c &=  1- \sqrt{1-O(\mu) + O(\mu^2)} - O(\mu^2)  \notag \\
	&= O(\mu)- O(\mu^2) = O(\mu) 
\end{align}
which allows us to rewrite the bound as follows:
\begin{align}
	\limsup_{i \to  \infty} \mathbb{E}\Vert \we_{c,i}\Vert^2 &\leq  O(\mu)\sigma_s^2  + O(\mu)\xi^2 + O(\mu^2)\sigma_q^2 
	\notag \\
	& + O(\mu^{-1})\mathbb{E}\Vert \bm{g}\Vert^2  + \sum_{p=1}^P O(1)\sigma_{g,p}^2 + O(\mu),
\end{align}
with $\mathbb{E}\Vert \bm{g}\Vert^2$ representing the variance of the total added noise, independent of time.
While in general decreasing the step-size improves performance, the above result shows that this need not be the case with privatization. Thus, since the added noise impacts the model utility negatively, it is important to choose a privatization scheme that reduces the effect. In what follows, we look closely at such a scheme.

\subsection{Graph Homomorphic Perturbations}
We consider a specific privatization scheme and specialize the above results. The goal of the scheme is to remove the $O(\mu^{-1})$ term from the MSE bounds. Thus, we wish to cancel out the total added noise amongst servers, i.e.:
\begin{equation}
	\sum_{p,m=1}^P a_{pm}\bm{g}_{pm,i} = 0.
\end{equation}
To achieve this, we introduce graph homomorphic perturbations \cite{vlaski2020graphhomomorphic} defined as follows. We assume each server $p$ draws a sample $\bm{g}_{p,i}$ independently from the Laplace distribution $Lap(0,\sigma_g/\sqrt{2})$ with variance $\sigma_{g}^2$. Server $p$ then sets the noise $\bm{g}_{mp,i}$ added to the message sent to its neighbour $m$ as:
\begin{equation}
	\bm{g}_{mp,i} = \begin{cases}
		\bm{g}_{p,i} & m \neq p \\
		- \frac{1-a_{pp}}{a_{pp}} \bm{g}_{p,i}.
	\end{cases}
\end{equation}
Furthermore, utilizing secret sharing methods, we can cancel out the noise added by clients at the server:
\begin{equation}
	\frac{1}{L}\sum_{k\in\Li{p}} \bm{g}_{p,k,i} = 0.
\end{equation}
One example of such methods is Diffie-Hellman key exchange \cite{DifHel}, where every pair of agents shares a secret key in order to securely communicate over a public channel. In the federated learning scenario, we wish to add noise to the messages sent to the server in such a way they cancel out at the server. Thus, we use the shared secret key among pairs of agents to generate the added noise that will cancel out at the server. To generate these shared secret keys, we assume we have a public generator $v$ and a prime modulous $t$. Then, each agent has a secret key, say agent 1 has $c$ and a public key $v^{c}$, which is broadcast. Next, each client raises all the public keys it receives from the other agents by its local secret key in order to get a shared secret key with every other agent. This allows every pair of agents to use their shared secret key as a seed to a pseudo-random generator to generate the noise.
	

Thus, with such a scheme, the noise components proportional to $O(\mu^{-1})$ cancel out in the error recursions, and the remaining error introduced by the noise is controlled by the step-size. Thus, its effect can be mitigated by using a smaller step-size.

\begin{corollary}[\textbf{Centroid MSE convergence under graph homomorphic perturbations}]\label{thrm:centMSE-ghp}
	Under Assumptions \ref{ass:adj}, \ref{ass:fct} and \ref{ass:mod}, the network centroid with graph homomorphic perturbations converges to the optimal point $w^o$ exponentially fast for a sufficiently small step-size $\mu$:
	\begin{align}
		\mathbb{E} \Vert \we_{c,i}\Vert^2 \leq& \lambda_c \mathbb{E} \Vert \we_{c,i-1} \Vert^2  +\mu^2 \sigma_s^2  + O(\mu^{2})\xi^2 + O(\mu^{3})\sigma_q^2 \notag \\
		&+ \frac{O(\mu)}{P}\sum_{p=1}^P \mathbb{E}\Vert \w_{p,i-1}-\w_{c,i-1}\Vert^2.
	\end{align}Then, letting $i$ tend to infinity, we get:
	\begin{align}
		\limsup_{i\to \infty} \mathbb{E} \Vert \we_{c,i}\Vert^2 \leq & \frac{\mu^2 \sigma_s^2 +O(\mu^{2})\xi^2 + O(\mu^{3})\sigma_q^2 }{1-\lambda_c}  \notag \\
		&+ \sum_{p=1}^PO(1 ) \sigma_{g,p}^2  + O(\mu).
	\end{align}
\end{corollary}
\begin{proof}
	Starting from \eqref{eq:thrmBd}, and replacing  $\mathbb{E} \Vert \bm{g}\Vert^2 = 0$ because $\bm{g}_{i} = 0$, we get the final result.
\end{proof}

\section{Privacy Analysis}
We study the privacy of the algorithm in terms of differential privacy. We focus on graph homomorphic perturbations and show that the adopted scheme is differentially private. To do so, we first define what it means for an algorithm to be $\epsilon-$differentially private. Therefore, without loss of generality, assume agent $1$ in federated unit $1$ decides to not participate, and its data samples $x_{1,1}$ are replaced by a new set $x'_{1,1}$ with a different distribution. Then, with the new data, the algorithm takes a different path. We denote the new models by $\w'_{p,k,i}$. The idea behind differential privacy is that an outside observant should not be able to distinguish between the two trajectories $\w_{p,k,i}$ and $\w'_{p,k,i}$ and conclude whether agent one participated in the training. More formally, differential privacy is defined bellow.

\begin{definition}[\textbf{$\epsilon(i)-$Differential Privacy}]\label{def:DP}
	We say that the algorithm given in \eqref{eq:privUpStep}--\eqref{eq:privCombStep} is $\epsilon(i)-$differentially private for server $p$ at time $i$ if the following condition holds: 
	\begin{equation}
		\frac{\mathbb{P}\left( \left\{   \left\{ \bm{\psi}_{p,j} + \bm{g}_{pm,j}   \right\}_{m\in \mathcal{N}_p \setminus \{p\} } \right\} _{j=0}^i \bigg | \mathcal{F}_{i} \right)}{\mathbb{P}\left( \left\{   \left\{ \bm{\psi}'_{p,j} + \bm{g}_{pm,j}   \right\}_{m\in \mathcal{N}_p \setminus \{p\} }\right\} _{j=0}^i  \bigg | \mathcal{F}_{i} \right)} \leq e^{\epsilon(i)}.
	\end{equation}
\qed
\end{definition}
Thus, the above definition states that minimaly varried trajectories have comparable probabilities. In addition, the smaller the value of $\epsilon$ is, the higher the privacy guarantee will be. Thus, the goal will be to decrease $\epsilon$ as long as the model utility is not strongly affected.

Next, in order to show that the algorithm is differentially private, we require the sensitivity of the alogorithm. It captures the effect of a change in an input variable on the output. Thus, we write:
\begin{align}
	\Delta(i) &= \max_{p,k} \Vert \w_{p,k,i} - \w'_{p,k,i}\Vert \notag \\
	&= \Vert \w_{1,1,i} - \w'_{1,1,i}\Vert \notag \\
	&= \mu \left\Vert \sum_{j=0}^{i-1}  \widehat{\grad{w}J_{1,1}} (\w_{1,j}) - \widehat{\grad{w}J_{1,1}} (\w'_{1,j})  \right\Vert  \notag \\
	&\leq 2\mu B i,
\end{align}
where the inequality follows from the triangle inequality and the below result in Lemma \ref{lem:bdGrad} on the bound of the gradients evaluated at the models resulting from the algorithm. Thus, a single change in the input data will change the trajectory by at most $2\mu B i$, where $B$ is some constant bound on the gradients. The sensitivity is also linearly proportional to the iteration; thus as time passes the two trajectories diverge.  

\begin{lemma}[\textbf{Bounded gradients along algorithm trajectory}]\label{lem:bdGrad}
	Assuming the loss functions $Q_{p,k}(\cdot;\cdot)$ satisfy Assumption \ref{ass:fct}, then they have bounded gradients along the trajectory of the algorithm: $\Vert \grad{w}Q_{p,k}(\w;x_{p,k,n}) \Vert \leq B,$
	for some constant B  that depends on the starting point of the algorithm $w_0$ and the local and global optimal models. 
\end{lemma}
\begin{proof}
	Proof found in Appendix \ref{app:lemBdGrad}.
\end{proof}

Using the bound on the sensitivity and from the definition of differential privacy, we can finally show that the algorithm is differentially private. 
\begin{theorem}[\textbf{Privacy of GFL algorithm}]\label{thrm:priv}
	If the algorithm \eqref{eq:privUpStep}-\eqref{eq:privCombStep} adopts graph homomorphic perturbations, then it is $\epsilon(i)-$differentially private, at time $i$ for a standard deviation of $\sigma_g = \sqrt{2}\mu B(1+i)i / \epsilon(i)$.
\end{theorem}

\begin{proof}
	Proof found in Appendix \ref{app:thrmPriv}.
\end{proof}

Thus, the above theorem suggests, if we wish the algorithm to be $\epsilon(i)-$differentially private, then we need to choose the noise variance accordingly. The larger the variance is, the more private the algorithm will be. However, the longer the algorithm is run, we will require a larger noise variance to keep the same level of privacy guarantee. Said differently, if we fix the added noise, then as time passes, the algorithm becomes less private, and more information is leaked. However, with graph-homomorphic perturbations, we can afford to increase the varianve since its effect is constant on the MSE, and thus decreases the leakage.

\section{Experimental Analysis}
We conduct a series of experiments to study the influence of privatization on the GFL algorithm. The aim of the experiments is to show the superior performance of graph homomorphic perturbations to random perturbations, and to study the effect of different parameters on the performance of the algorithm.

\subsection{Regression }
We first start by studying a regression problem on simulated data. We do so for the tractability of the problem. We consider the quadratic loss that has a closed form solution, i.e., a formal expression for the true model $w^o$ is known, which makes the calculation of the mean square error feasible and more accurate.

Therefore, consider  a streaming feature vector $\bm{u}_{p,k,n} \in \mathbb{R}^M$ with output variable $\bm{d}_{p,k}(n) \in \mathbb{R}$ given by:
\begin{equation}\label{eq:LMSdata}
	\bm{d}_{p,k}(n) = \bm{u}_{p,k,n}^\tran w^{\star} + \bm{v}_{p,k}(n),
\end{equation}
where $w^{\star}\in \mathbb{R}^M$ is some generating model, and $\bm{v}_{p,k}(n)$ is some zero-mean Guassian random variable with $\sigma_{v_{p,k}}^2$ variance and independent of $\bm{u}_{p,k,n}$. Then, the optimal model that solves the following problem:
\begin{equation}
	\min_w \frac{1}{P}\sum_{p=1}^P \frac{1}{K}\sum_{k=1}^K \frac{1}{N_{p,k}}\sum_{n=1}^{N_{p,k}} \Vert \bm{d}_{p,k}(n) - \bm{u}_{p,k,n}^\tran w\Vert^2  + \rho \Vert w\Vert^2
\end{equation}
is found to be:
\begin{equation}
	w^o = (\widehat{R}_u + \rho I)^{-1}( \widehat{R}_u w^{\star} + \widehat{r}_{uv}),
\end{equation}
where $\widehat{R}_u$ and $\widehat{r}_{uv}$ are defined as:
\begin{align}
	\widehat{R}_u &\eqdef  \frac{1}{P}\sum_{p=1}^P \frac{1}{K}\sum_{k=1}^K \frac{1}{N_{p,k}} \sum_{n=1}^{N_k} \bm{u}_{p,k,n}\bm{u}_{p,k,n}^\tran, \label{eq:dfRu}\\
	\widehat{r}_{uv} & \eqdef \frac{1}{P}\sum_{p=1}^P \frac{1}{K}\sum_{k=1}^K \frac{1}{N_{p,k}} \sum_{n=1}^{N_k} \bm{v}_{p,k}(n)\bm{u}_{p,k,n}. \label{eq:dfruv}
\end{align}

We consider $P = $10 units, each with $K= 100$ total agents. We assume, $N_{p,k}=100$ for each agent. We randomly generate two-dimensional feature vectors $\bm{u}_{p,k}(n) $ from a Guassian random vector with zero-mean and a randomly generated covarinace matrix $R_{u_{p,k}}$. We then calculate the corresponding outputs according to \eqref{eq:LMSdata}. To make the data non-iid accross agents, we assume the covariance matrix $R_{u_{p,k}}$ is different for each agent, as well as the variance $\sigma_{v_{p,k}}^2$ of the added noise. When running the algorithm, we assume each unit samples at random $L = 11$ agents, and each agent runs $E_{p,k} \in [1,10]$ epochs and uses a mini-batch of $B_{p,k} \in [5,10]$ samples. 

We compare three algorithms: the standard GFL algorithm, the privatized GFL algorithm with random perturbations, and the privatized GFL with homomorphic perturbations. In the first set of simulations, we fix the step-size $\mu =0.7 $ and the regularization parameter $\rho = 0.1$. We fix the variance of the added noise for privatization in both schemes to $\sigma_g^2 = 0.1 $. We then plot the mean-square-deviation at each time step for the centroid model:
\begin{equation}
	\mbox{\rm MSD}_i \eqdef \Vert \w_{c,i} - w^o\Vert^2,
\end{equation}
as seen in Figure \ref{fig:3Alg}. We observe that the privatized GFL with random perturbations has lower performance compared to the other two algorithms. While, using homomorphic perturbations does not result in such a decay in performance. Thus, our suggested scheme does a good job at tracking the performance of the original GFL algorithm, while not compromising with the privacy level. 


\begin{figure}[h]
	\centering
	\includegraphics[width=0.5\textwidth]{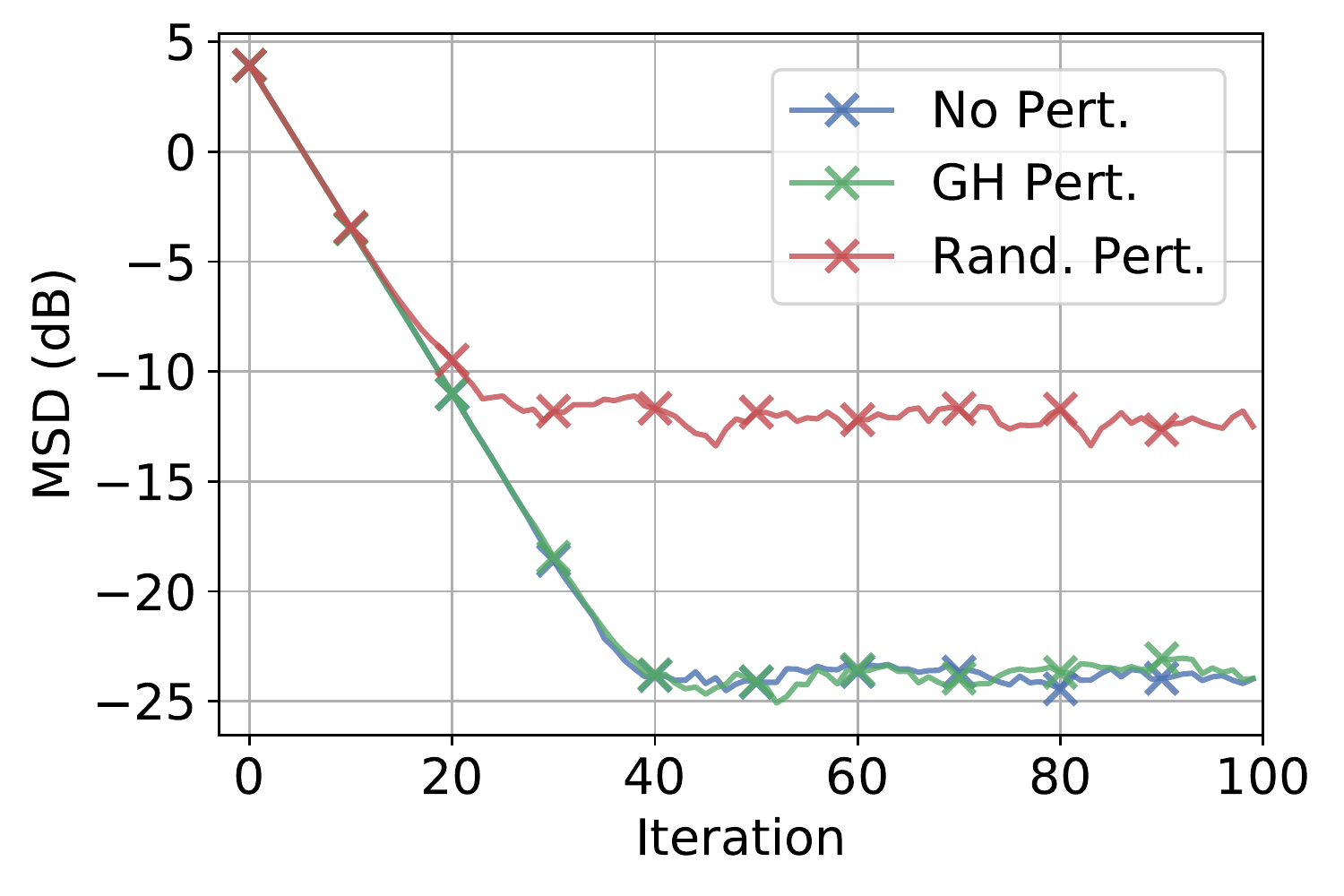}
	\caption{Performance of GFL with no perturbations (blue), with graph homomorphic perturbations (green), and random perturbations (red).}\label{fig:3Alg}
\end{figure}

\begin{figure*}[h]
	\centering
	\centering
	\begin{subfigure}[b]{\textwidth}
		\centering
		\includegraphics[width=\textwidth]{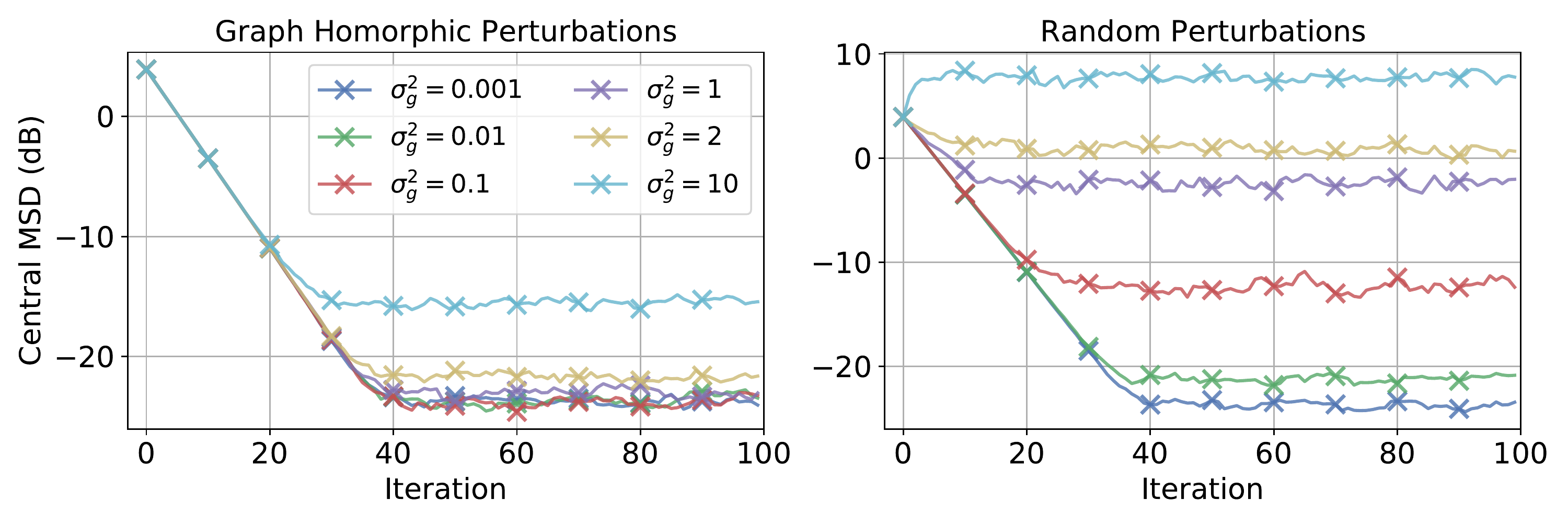}
		\caption{centroid model}\label{fig:varNoiseCent}
	\end{subfigure}
\\ \hfill
	\begin{subfigure}[b]{\textwidth}
		\centering
		\includegraphics[width=\textwidth]{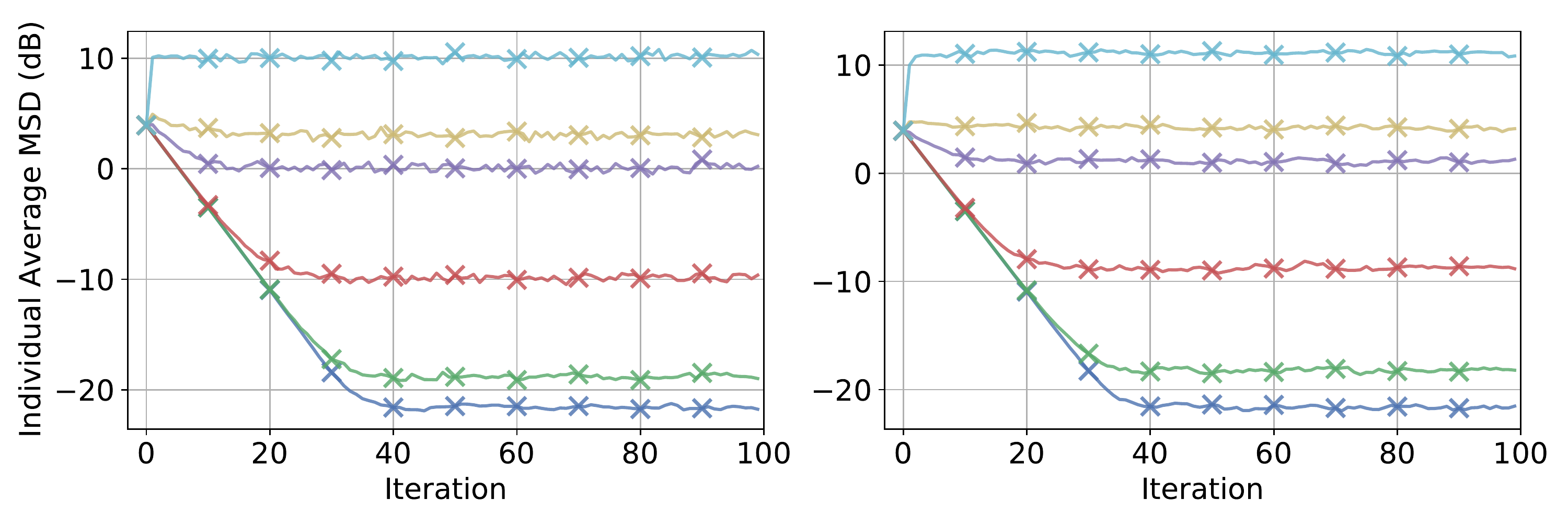}
		\caption{individual models }\label{fig:varNoiseInd}
	\end{subfigure}
	\caption{Performance curves of privatized GFL with varying noise variance.}
\end{figure*}

We next study the extent of the effect of the noise on the model utility. Thus, we run a series of experiments with varying added noise $\sigma_g^2 = \{0.001, 0.01, 0.1,1,2,10\}$ for the two privatized GFL algorithms. We plot the resulting MSD curves in Figure \ref{fig:varNoiseCent}. We obsereve for a fixed step-size, as we increase the variance, the MSD of the algorithm with random perturbations increases significantly as opposed to the algorithm with homomorphic perturbations. Thus, we conclude that the algorithm with random perturbtaions is more sensitive to the variance of the added noise. In fact, at some point, while using random perturbations, for some variance, the algorithm breaks down. While using graph homomorphic perturbations, delays that effect for much larger variance. In addition, as long as the step-size is small enough, we can always control the effect of the graph homomorphic perturbations.

However, if we were to look at the individual MSD for one federated unit, we would discover that the performance of the algorithm decays as the noise variance is increased. Nonetheless, it is not to the extent of random perturbations. We plot in Figure \ref{fig:varNoiseInd} the average individual MSD for the varying noise variance: \vspace*{-0.3cm}
	\begin{equation}
		\mbox{\rm MSD}_{\mbox{\rm avg},i} \eqdef \frac{1}{P}\sum_{p=1}^P \Vert \w_{p,i} - w^o\Vert^2.
	\end{equation}
	We observe that for a fixed noise variance, homomorphic perturbations results in a better performance.  Furthermore, as we increase the noise variance, the network disagreement increases for both schemes. This comes as no surprise and is in accordance with Lemma \ref{lem:netDis}. Furthermore, as previously mentioned, graph homomorphic perturbations have the added value of not being negatively affected by the decrease in the step-size. In addition, even though the improvement does not seem significant, the source of the error of the two schemes is different. Furthemore, the information of the true model is distributed in the network and can be retrieved by running at the end of the learning algorithm a consensus-type step. At that point, the local models no longer contain information about the local data, and thus agents can safely share their models. However, when random perturbations are used, reconstruction is not possible since the information has been lost in the netwrok due to the added perturbations. 

	\begin{figure*}[h!]
	\centering
	\includegraphics[width=\textwidth]{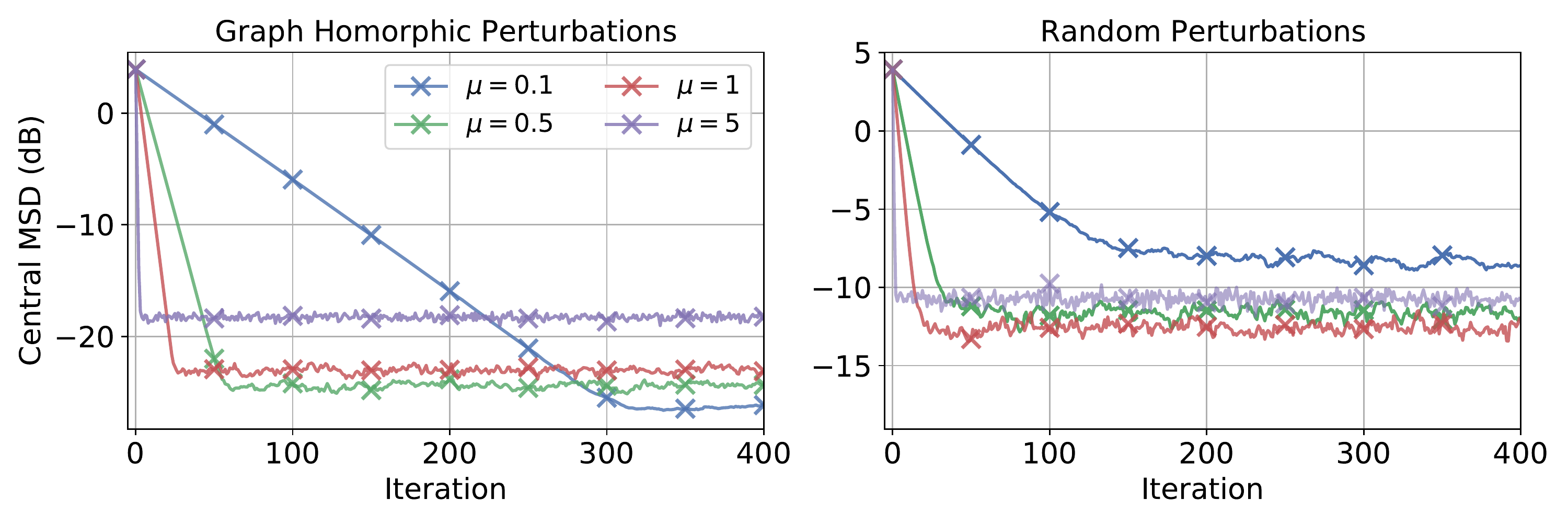}
	\caption{Performance curves of privatized GFL with varying step-size.}\label{fig:varMu}
\end{figure*}
	
	We next fix the noise variance $\sigma_{g}^2 = 0.1$ and varying the step-size $\mu = \{0.1, 0.5, 1, 5  \}$. According to Theorem \ref{thrm:indMSE}, the MSD resulting from random perturbations includes an $O(\mu^{-1})$ term, which is not the case when using graph homomorphic perturbations. Thus, we expect a decrease in the step-size will not significantly affect the privatized algorithm with graph homomorphic perturbations as opposed to random perturbations. Indeed, as seen in Figure \ref{fig:varMu}, as $\mu$ is increased, the final MSD increases; this is probably due to the $O(\mu)\sigma_s^2$ term in the bound. In contrast, for significantly small or large $\mu$, the performance of the privatized algorithm with random perturbations decreases. In addition, what we observe for both privacy schemes, is that the rate of convergence slows down as we decrease the step-size. Thus, there exists an optimal step-size that achieves a good compromise between a fast convergence and a low MSD.


\subsection{Classification}
We now focus on a classification problem applied to a dataset on click rate prediction of ads. We consider the Avazu click through dataset \cite{avazu}. We split the 5101 data unequally among a total of 50 agents. We assume there are $P = 5$ units each with $K = 10$ agents. We add non-idd noise to the data at each agent to change their distributions. We again compare three algorithms: standard GFL, privatized GFL with homomorphic perturbations, and privatized GFL with random perturbations. We use a regularized logistic risk with regularization parameter $\rho = 0.03$. We set the step-size $\mu = 0.5$. We repeat the algorithms for multiple levels of privacy. We then settle on a noise variance $\sigma_g^2 = 0.6$ for which the privatized algorithm with random perturbations still converges. We plot in Figure \ref{fig:Avazu} the testing error on a set of 256 clean samples that were not perturbed with noise to change their distributions. We use the centriod model learned during each iteration to calculate the corresponding testing error. We observe that the graph homorphic perturbations do not hinder the performance of the privatized model. As for random perturbations, they significantly reduce the utility of the learnt model. 
 
\begin{figure}[h]
	\centering
	\includegraphics[width=0.5\textwidth]{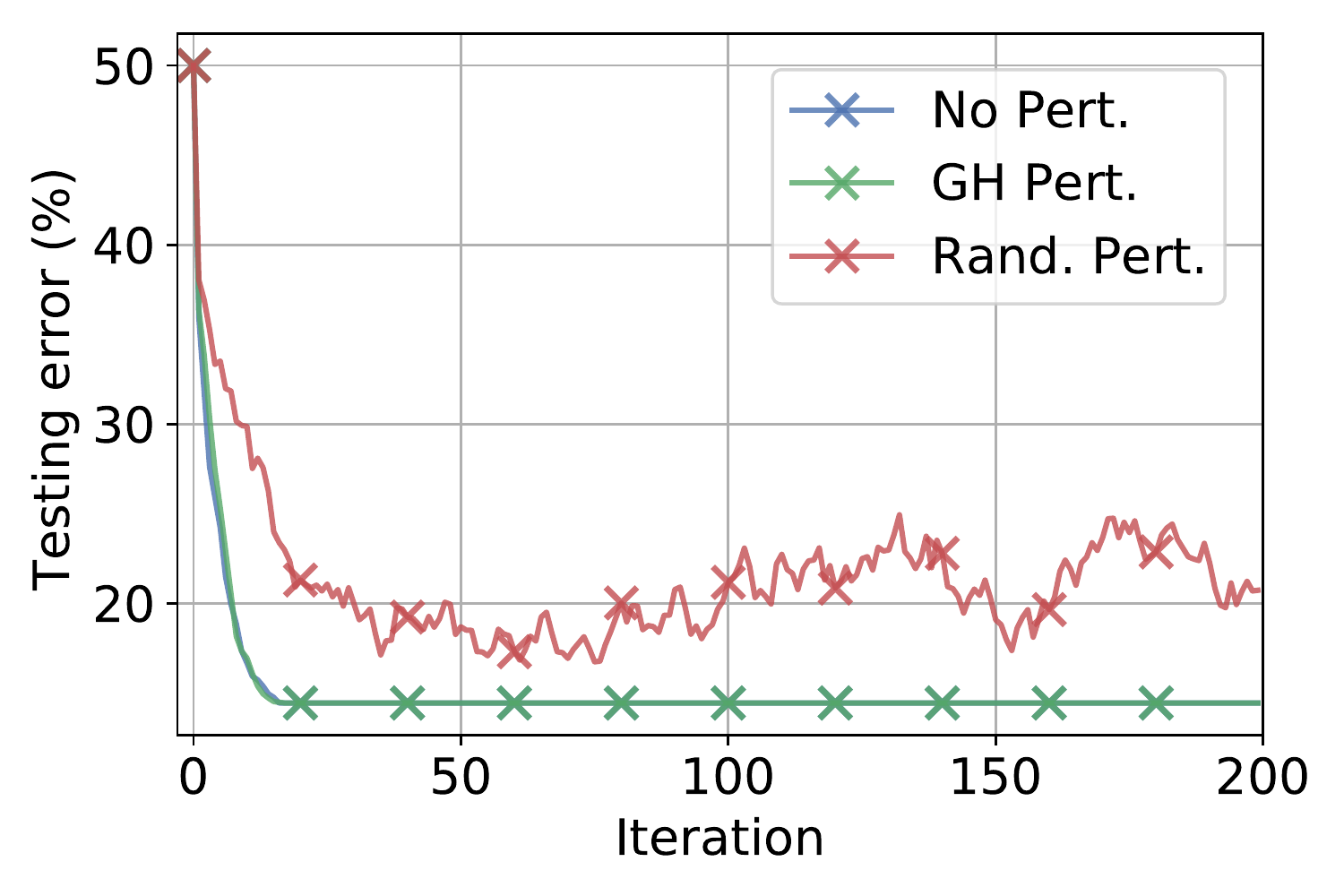}
	\caption{Testing error of GFL with no perturbations (blue), with graph homomorphic perturbations (green), and random perturbations (red).}\label{fig:Avazu}
\end{figure}

\vspace*{-0.2cm}
\section{Conclusion}
In this work, we introduced graph federated learning and implemented an algorithm that guarantees privacy of the data in a differential privacy sense. We showed general privatization based on adding random perturbations to updates in federated learning have a negative effect on the performance of the algorithm. Random perturbations drive the algorithm farther away from the true optimal model. However, we showed by adding graph homomorphic perturbations, which exploit the graph structure, performance can be recovered with guaranteed privacy. We also showed that using dependent perturbations does not result in the same trade-off between privacy and efficiency. Thus, we no longer have to choose what to prioritize, and instead, we can have both a highly privatized algorithm with a good model utility.

\appendices
\section{Secondary Result: Individual MSE Performance}\label{app:thrmIndMSE}
We first introduce the following theorem, which will be used to bound the network disagreement. We loosely bound the individual MSE for each federated unit. A tighter bound can be found, however, it is not needed. 

\begin{theorem}[\textbf{Individual MSE convergence}]\label{thrm:indMSE}
	Under Assumptions \ref{ass:adj}, \ref{ass:fct} and \ref{ass:mod}, the individual models converge to the optimal model $w^o$ exponentially fast for a sufficiently small step-size:
	\begin{align}
		&\mbox{\rm col}\{\mathbb{E} \Vert \we_{p,i} \Vert^2 \}_{p=1}^P \notag \\
		&\preceq \Lambda^i  \mbox{\rm col}\{\mathbb{E}\Vert \we_{p,0} \Vert^2 \}_{p=1}^P \notag \\
		&+  \sum_{j=0}^i \Lambda^j  \mbox{\rm col}\{\mu^2\sigma_{s,p}^2+ O(\mu^2)\xi^2 + O(\mu^3)\sigma_{q,p}^2 + \sigma_{g,p}^2 \}_{p=1}^P ,
	\end{align}
where $\preceq$ is the elementwise comparison, $\Lambda$ is a diagonal matrix with the $p^{th}$ entry given by $\lambda_p = \sqrt{1-2\nu\mu + \delta^2\mu^2} + \beta_{s,p}^2\mu^2 + O(\mu^2)\in (0,1)$, $\sigma_{q,p}^2$ the average of $\sigma_{q,p,k}^2$, and $\sigma_{g,p}^2$ is the total variance introduced by the noise added at server $p$. Then, taking the limit of $i$ to infinity:

\begin{align}
&	\limsup_{i\to \infty} \col{\mathbb{E} \Vert \we_{p,i} \Vert^2 }_{p=1}^P\notag \\ 
& \preceq (I - \Lambda)^{-1}  
\col{\mu^2\sigma_{s,p}^2 + O(\mu^2)\xi^2 + O(\mu^3)\sigma_{q,p}^2 + \sigma_{g,p}^2 }_{p=1}^P.
\end{align}
\end{theorem}

\begin{proof}
	Focusing on the error of a single server $p$, we can verify that:
	\begin{align}
	&	\mathbb{E} \{\Vert \we_{p,i}\Vert^2 | \mathcal{F}_{i-1} \}\notag \\
		& \stackrel{(a)}{=} \mathbb{E} \Bigg\{\bigg\Vert \sum_{m \in \mathcal{N}_p} a_{pm} \big( \we_{m,i-1}  + \mu \grad{w}J_m(\w_{m,i-1}) \notag \\
		&\quad + \mu \bm{q}_{m,i} \big) \bigg\Vert^2  \bigg |\mathcal{F}_{i-1} \Bigg\}\notag \\
		&\quad + \mu^2 \mathbb{E}\left\{ \bigg\Vert \sum_{m\in \mathcal{N}_p} a_{pm} \bm{s}_{m,i} \bigg\Vert^2 \bigg| \mathcal{F}_{i-1} \right\} \notag \\
		&\quad +\mathbb{E} \left\{\bigg \Vert \sum_{m \in \mathcal{N}_p}  \frac{a_{pm}}{L} \sum_{k\in\Li{m}} \bm{g}_{m,k,i} \bigg\Vert^2 \bigg| \mathcal{F}_{i-1} \right\} \notag \\
		&\quad + \mathbb{E} \left\{ \bigg\Vert \sum_{m\in \mathcal{N}_p} a_{pm}\bm{g}_{pm,i} \bigg\Vert^2 \bigg | \mathcal{F}_{i-1} \right\}, \notag \\
		&\stackrel{(b)}{\leq}  \sum_{m\in \mathcal{N}_p} a_{pm} \Bigg(\frac{1}{\alpha} \Vert \we_{m,i-1} + \mu \grad{w}J_m(\w_{m,i-1})\Vert^2 \notag \\
		&\quad + \frac{\mu^2}{1-\alpha}\left(O(\mu)\Vert \we_{m,i-1}\Vert^2 + O(\mu) \xi^2 + O(\mu^2) \sigma_{q,m}^2 \right)  \notag \\
		&\quad +  \mu^2  \left(\sigma_{s,m}^2+\beta_{s,m}^2\Vert  \we_{m,i-1}\Vert^2 \right) + \frac{1}{LK}\sum_{k=1}^K \mathbb{E}\Vert \bm{g}_{m,k,i}\Vert^2 \notag \\
		&\quad + \mathbb{E}\Vert \bm{g}_{pm,i}\Vert^2  \Bigg), \notag
		 \\	
		&\stackrel{(c)}{\leq} \sum_{m\in \mathcal{N}_p}a_{pm} \Bigg( \bigg( \frac{1-2\nu\mu + \delta^2 \mu^2 }{\alpha}  + \beta_{s,m}^2\mu^2  \notag \\
		&\quad + \frac{O(\mu^3)}{1-\alpha} \bigg) \Vert \we_{m,i-1}\Vert^2  + \mu^2 \sigma_{s,m}^2+ \frac{O(\mu^3)\xi^2 + O(\mu^4)\sigma_{q,m}^2}{1-\alpha}  \notag \\
		& \quad  + \frac{1}{LK}\sum_{k=1}^K \mathbb{E}\Vert \bm{g}_{m,k,i}\Vert^2 + \mathbb{E}\Vert \bm{g}_{pm,i}\Vert^2 \Bigg),
	\end{align}
where we define $\sigma_{q,m}^2$ to be the average of $\sigma_{q,m,k}^2$. Step $(a)$ follows from independence of random variables and the zero-mean of the gradient noise and the added noise, $(b)$ from Jensens' inequality and the bound on the gardient noise \eqref{eq:bdSerGradNoise} and the incremental noise \eqref{eq:lemQBd}, $(c)$ from $\nu$-strong convexity and $\delta$-Lipschtz continuity.
Then, choosing $\alpha = \sqrt{1-2\nu\mu + \delta^2 \mu^2} = 1-O(\mu)$ and taking the expectation over the filtration, we get:
\begin{align}
	\mathbb{E}\Vert \we_{p,i}\Vert^2 &\leq \sum_{m\in \mathcal{N}_p}a_{pm}\left( \lambda_m \mathbb{E}\Vert \we_{m,i-1}\Vert^2 + \mu^2\sigma_{s,m}^2 + O(\mu^2)\xi^2 \right. \notag \\
	&\quad \left.+ O(\mu^3)\sigma_{q,m}^2 + \sigma_{g,m}^2 \right), 
\end{align}
where we introduce the constants $\lambda_m$ and  $\sigma_{g,m}^2$:
\begin{equation}
	\lambda_m \eqdef \sqrt{1-2\nu\mu +\delta^2\mu^2} + \beta_{s,m}^2 \mu^2 + O(\mu^2).
\end{equation}
 Next, taking the column vector of every local mean-square-error, we get the following bound in which we drop the indexing from the column vectors:
\begin{align}
	&\col{\mathbb{E}\Vert \we_{p,i}\Vert^2} \notag \\
	 &\preceq \Lambda A \:	\col{\mathbb{E}\Vert \we_{p,i-1}\Vert^2}+ A \: \col{\mu^2\sigma_{s,p}^2  + \sigma_{g,p}^2+ O(\mu^2)\xi^2}
	 \notag \\
	 	&\quad + A\:\col{  O(\mu^3)\sigma_{q,p}^2 }, 
	 	\notag \\
	&\preceq \Lambda^i A^i \col{\mathbb{E}\Vert \we_{p,0}\Vert^2} + \sum_{j=0}^i \Lambda^j A^j \col{\mu^2\sigma_{s,p}^2 + \sigma_{g,p}^2 } \notag \\
	&\quad + \Lambda^jA^j\col{  O(\mu^2)\xi^2+O(\mu^3)\sigma_{q,p}^2 }, \notag \\
	&\preceq \Lambda^i \col{\mathbb{E}\Vert \we_{p,0}\Vert^2} + \sum_{j=0}^i \Lambda^j \col{\mu^2\sigma_{s,p}^2 + \sigma_{g,p}^2 + O(\mu^2)\xi^2} \notag \\
	&\quad + \Lambda^j\col{ O(\mu^3)\sigma_{q,p}^2 },
\end{align}
where we define the diagonal matrix $\Lambda$ with $\lambda_p$ as entries on the diagonal.
Then choosing $\mu$ small enough such that $\lambda_p < 1$ for every $p$, we know the limit of $\Lambda^i$ as $i$ goes to infinity is zero. Furthermore, if the eigenvalues of $\Lambda$ are less than 1, which they are, then the geometric series converges to $(I - \Lambda )^{-1}$.  Thus, we get the desired result.

\end{proof}

\section{Proof of Lemma \ref{lem:netDis}}\label{app:lemmNetDis}
	Consider the aggregate model vector, i.e., $\ws_{i} \eqdef \col{\w_{p,i}}_{p=1}^P$, for which we write the model recursion as:
	\begin{align}\label{eq:aggModRec}
		\ws_i = &(A \otimes I)^\tran \Bigg(\ws_{i-1} - \mu \col{\grad{w}J_p(\w_{p,i-1}) + \bm{s}_{p,i} + \bm{q}_{p,i}}  \notag \\
		&  \left. +  \col{\frac{1}{L}\sum_{k\in \Li{p} } \bm{g}_{p,k,i}} \right) +  \text{diag}\big( (A \otimes I)^\tran \bm{\mathcal{G}}_i \big), 
	\end{align}
	where $\bm{\mathcal{G}}_i$ is a matrix whose entries are the noise $\bm{g}_{pm,i}$, and the diag$(\cdot)$ function extracts the diagonal entries of a matrix and transforms them into a column vector. 

	Since $A$ is doubly-stochastic, then it admits an eigendecomposition of the form $A = QH Q^\tran$, with the first eigenvalue equal to 1 and its corresponding eigenvector equal to $\mathds{1}/\sqrt{P}$.
	
	Next, we define the extended centroid model $\ws_{c,i} \eqdef \left(\frac{1}{P} \mathds{1}\mathds{1}^\tran \otimes I\right) \ws_{i}$, and write:
	\begin{align}
		\ws_i - \ws_{c,i} 
		&= \left(I - \frac{1}{P}\mathds{1}\mathds{1}^\tran \otimes I \right)\ws_{i} \notag \\
		&= \left( (Q^{\tran}\otimes I ) (Q\otimes I ) - \frac{1}{P}\mathds{1}\mathds{1}^\tran \otimes I \right)\ws_{i}	 \notag \\
		&= (Q_{\epsilon}^\tran \otimes I)(Q_{\epsilon} \otimes I) \ws_i \notag \\
		&= (Q_{\epsilon}^\tran \otimes I)H_{\epsilon} (Q_{\epsilon} \otimes I) \Bigg(  \ws_{i-1} \notag \\
		&\quad -\mu\col{ \grad{w}J_p(\w_{p,i-1}) + \bm{s}_{p,i} + \bm{q}_{p,i}}  \notag \\
		&\quad + (Q_{\epsilon}^\tran \otimes I )(Q_{\epsilon} \otimes I)\text{diag} \left(  (A\otimes I)^\tran \bm{\mathcal{G}}_i \right) \notag
		\end{align}
	\begin{align}
		&\quad \left. + \col{\frac{1}{L} \sum_{k \in \Li{p}} \bm{g}_{p,k,i}} \right)  .
	\end{align}
Then, taking the conditional expectation given the past models of $\Vert (Q_{\epsilon} \otimes I)\ws_{i}\Vert^2$, we can split the gradient noise and the added privacy noise from the model and the true gradient. Taking again the expectation over the past data, and then using the sub-multiplicity property of the norm followed by Jensen's inequality, we have:
\begin{align}
	&\mathbb{E} \Vert (Q_{\epsilon} \otimes I) \ws_{i}\Vert^2  \notag \\
	&\leq  \Vert H_{\epsilon}\Vert^2\Bigg( \mathbb{E} \left\Vert  
	(Q_{\epsilon} \otimes I)  \ws_{i-1} \right. \notag \\
	& \left. - 
	(Q_{\epsilon} \otimes I) \mu\col{ \grad{w}J_p(\w_{p,i-1}) + \bm{q}_{p,i} } \right\Vert^2  
	 \notag \\
	& + \mu^2\Vert Q_{\epsilon} \otimes I\Vert^2 \sum_{p=1}^P\mathbb{E} \Vert  \bm{s}_{p,i}\Vert^2
	 \notag \\
	&  \left. + \Vert Q_{\epsilon} \otimes I \Vert^2\sum_{p=1}^P \mathbb{E} \left\Vert \frac{1}{L}\sum_{k\in \Li{p}} \bm{g}_{p,k,i}\right\Vert^2  \right)
	\notag \\
	& + \Vert Q_{\epsilon} \otimes I \Vert^2 \mathbb{E} \Vert \text{diag}\left((A\otimes I)^\tran \bm{\mathcal{G}}_i \right)\Vert^2
	 \notag \\
	&\leq  \Vert H_{\epsilon}\Vert^2 \Bigg( \frac{1}{\Vert H_{\epsilon}\Vert }   \mathbb{E} \Vert (Q_{\epsilon} \otimes I)\ws_{i-1}\Vert^2
	 \notag \\
	& + \frac{\mu^2 \Vert Q_{\epsilon} \otimes I \Vert^2}{1-\Vert H_{\epsilon}\Vert  }  \sum_{p=1}^P  \mathbb{E} \Vert \grad{w}J_p(\w_{p,i-1}) + \bm{q}_{p,i}\Vert^2 
	\notag \\
	&  + \mu^2 \Vert Q_{\epsilon} \otimes I \Vert^2 \sum_{p=1}^P \mathbb{E} \Vert \bm{s}_{p,i}\Vert^2 
	\notag \\
	& \left. + \Vert Q_{\epsilon} \otimes I \Vert^2 \sum_{p=1}^P \mathbb{E} \left\Vert \frac{1}{L}\sum_{k\in \Li{p}} \bm{g}_{k,p,i} \right\Vert^2  \right) 
	\notag \\
	&+\Vert Q_{\epsilon} \otimes I\Vert^2  \mathbb{E}\Vert \text{diag} \left( (A\otimes I)^\tran \bm{\mathcal{G}}_i \right) \Vert^2 .
\end{align}
Next, we focus on each individual term. Using Jensen for some constant $\alpha$ and then the Lipschitz condition and the bound on the incremental noise, we can bound the below norm as follows:
\begin{align}
	&\mathbb{E} \Vert \grad{w}J_p(\w_{p,i-1}) + \bm{q}_{p,i}\Vert^2 \notag \\
	&\leq \frac{2}{\alpha} \left( \delta^2 \mathbb{E}\Vert \we_{p,i-1}\Vert^2 + \Vert \grad{w}J_p(w^o)\Vert^2 \right)\notag \\
	&+ \frac{1}{1-\alpha } \left(O(\mu)\mathbb{E}\Vert \we_{p,i-1}\Vert^2 + O(\mu)\xi^2 + O(\mu^2) \sigma_{q,p}^2 \right).
\end{align}
Using  the bound on the gradient noise \eqref{eq:bdSerGradNoise}, we get another $\mathbb{E} \Vert \we_{p,i-1}\Vert^2$ term, which can be bounded by the result in Theorem \ref{thrm:indMSE}. Thus, we write:
\begin{align}
& 	\frac{1}{1-\Vert H_{\epsilon}\Vert} \mathbb{E}\Vert \grad{w}J_p(\w_{p,i-1}) + \bm{q}_{p,i}\Vert^2 +  \mathbb{E}\Vert \bm{s}_{p,i}\Vert^2 \notag \\
& \leq  \left( \frac{2\delta^2 }{\alpha(1- \Vert H_{\epsilon}\Vert) } +\beta_{s,p}^2 + \frac{O(\mu)}{(1-\alpha)(1-\Vert H_{\epsilon}\Vert)} \right)\mathbb{E}\Vert \we_{p,i-1}\Vert^2 \notag \\
&\quad + \frac{ 2\Vert \grad{w}J_p(w^o)\Vert^2}{\alpha(1-\Vert H_{\epsilon}\Vert)}+ \sigma_{s,p}^2 + \frac{O(\mu)\xi^2 + O(\mu^2)\sigma_{q,p}^2}{(1-\alpha)(1-\Vert H_{\epsilon}\Vert)}   \notag 
\end{align}
\begin{align}
&\leq \left( \frac{2\delta^2 }{\alpha(1- \Vert H_{\epsilon}\Vert )} +\beta_{s,p}^2 + \frac{O(\mu)}{(1-\alpha)(1-\Vert H_{\epsilon}\Vert)} \right) \notag \\
&\quad \times\Bigg( \lambda_p^i A^i[p] \col{\mathbb{E}\Vert \we_{p,0}\Vert^2} + \sum_{j=0}^{i-1}\lambda_p^j A^j[p]  \notag \\
& \quad  \times \col{\mu^2 \sigma_{s,p}^2+O(\mu^2)\xi^2 + O(\mu^3)\sigma_{q,p}^2+ \sigma_{g,p}^2} \Bigg)  \notag \\
& \quad + \frac{2\Vert \grad{w}J_p(w^o)\Vert^2}{\alpha(1-\Vert H_{\epsilon}\Vert) }+ \sigma_{s,p}^2 + \frac{O(\mu)\xi^2 + O(\mu^2)\sigma_{q,p}^2}{(1-\alpha)(1-\Vert H_{\epsilon}\Vert)} . 
\end{align}
The noise term can be witten in a more compact way,  $ \qquad\Vert Q_{\epsilon} \otimes I \Vert^2 \sum\limits_{p=1}^P \sigma_{g,p}^2.$ Thus, putting everything together, we get:
\begin{align}
	& \mathbb{E}\Vert(Q_{\epsilon} \otimes I) \ws_i\Vert^2 \notag \\
	&\leq \Vert H_{\epsilon}\Vert \mathbb{E} \Vert (Q_{\epsilon} \otimes I) \ws_{i-1}\Vert^2 + \mu^2 \Vert Q_{\epsilon} \otimes I \Vert^2 \Vert H_{\epsilon}\Vert^2  \notag \\
	&\quad \times \sum_{p=1}^P \Bigg( \left( \frac{2\delta^2}{\alpha(1-\Vert H_{\epsilon}\Vert) } + \beta_{s,p}^2 + \frac{O(\mu)}{(1-\alpha)(1-\Vert H_{\epsilon}\Vert)}\right)
	 \notag \\
	&\quad \times \Bigg( \lambda_p^i A^i[p] \col{\mathbb{E}\Vert \we_{p,0}\Vert^2}+  \sum_{j=0}^{i-1}\lambda_p^j A^j[p] 
	\notag \\
	& \quad \times \col{\mu^2 \sigma_{s,p}^2 + O(\mu^2)\xi^2 + O(\mu^3)\sigma_{q,p}^2 + \sigma_{g,p}^2 } \Bigg)
	 \notag \\
	& \quad+ \frac{2\Vert \grad{w}J_p(w^o)\Vert^2}{\alpha(1-\Vert H_{\epsilon}\Vert )} + \sigma_{s,p}^2 + \frac{O(\mu)\xi^2 + O(\mu^2)\sigma_{q,p}^2}{(1-\alpha)(1-\Vert H_{\epsilon}\Vert )}\Bigg)
	 \notag \\
	 &\quad +  \Vert Q_{\epsilon} \otimes I \Vert^2 \sum_{p=1}^P \sigma_{g,p}^2 \notag 
	\\	
	& \leq \Vert H_{\epsilon}\Vert^i \mathbb{E} \Vert (Q_{\epsilon} \otimes I) \ws_{0}\Vert^2 + \sum_{j'=0}^{i-1} \Vert H_{\epsilon}\Vert^{j'+2} \Vert Q_{\epsilon} \otimes I \Vert^2 \Bigg \{    \mu^2   \notag \\
	&\quad \times \sum_{p=1}^P \Bigg( \left( \frac{2\delta^2}{\alpha(1-\Vert H_{\epsilon}\Vert) } + \beta_{s,p}^2 + \frac{O(\mu)}{(1-\alpha)(1-\Vert H_{\epsilon}\Vert)}\right) \notag 
	\\
	&\quad \times \Bigg( \lambda_p^{j'} A^{j'}[p] \col{\mathbb{E}\Vert \we_{p,0}\Vert^2} + \sum_{j=0}^{j'-1}\lambda_p^j A^j[p]  \notag \\
	& \quad \times  \col{\mu^2 \sigma_{s,p}^2 +O(\mu^2)\xi^2 + O(\mu^3)\sigma_{q,p}^2 + \sigma_{g,p}^2} \Bigg) \notag \\
	& \quad+ \frac{2\Vert \grad{w}J_p(w^o)\Vert^2}{\alpha(1-\Vert H_{\epsilon}\Vert )} + \sigma_{s,p}^2 + \frac{O(\mu)\xi^2 + O(\mu^2)\sigma_{q,p}^2}{(1-\alpha)(1-\Vert H_{\epsilon}\Vert )}\Bigg) \notag
\\
&	\quad  +  \frac{1}{\Vert H_{\epsilon}\Vert^2 } \sum_{p=1}^P \sigma_{g,p}^2 \Bigg\}.
\end{align}
Going back to the network disagreement, it is bounded by the above bound multiplied by $\Vert Q_{\epsilon}^\tran \otimes I\Vert^2/P$. If we were to drive $i$ to infinity, since $\Vert H_{\epsilon}\Vert = \iota_2< 1$, with $\iota_2$ being the second eigenvalue of $A$, and choosing $\alpha = \iota_2$ we would have:
\begin{align}
	&\limsup_{i\to\infty}\frac{1}{P}\sum_{p=1}^P \mathbb{E}\Vert \w_{p,i} - \w_{c,i}\Vert^2 \notag 
\end{align}
\begin{align}
	 &\leq \frac{\Vert Q_{\epsilon} \otimes I \Vert^4 \iota_2^2 }{P} \Bigg\{ \mu^2 \sum_{p=1}^P \Bigg( \left( \frac{2\delta^2}{\iota_2(1- \iota_2)} + \beta_{s,p}^2   \right. 
	 \notag \\
	 &\quad \left. + \frac{O(\mu)}{(1-\iota_2)^2}\right)  \sum_{j'=0}^{\infty}\iota_2^{j'}\sum_{j=0}^{j'-1} \lambda_p^jA^j[p]  
	 \notag \\
	& \quad \times \col{\mu^2 \sigma_{s,p}^2 + O(\mu^2)\xi^2 + O(\mu^3)\sigma_{q,p}^2+ \sigma_{g,p}^2}   
	\notag \\
	& \quad+ \frac{2\Vert \grad{w}J_p(w^o)\Vert^2}{\iota_2(1-\iota_2)^2 }  + \frac{\sigma_{s,p}^2 }{1-\iota_2}+ \frac{O(\mu)\xi^2 + O(\mu^2)\sigma_{q,p}^2}{(1-\iota_2)^3}\Bigg) \notag \\
	&\quad + \frac{1}{(1-\iota_2)\iota_2^2}\sum_{p=1}^P \sigma_{g,p}^2 \Bigg\} 
	\notag \\
	&\leq \frac{\iota_2^2}{P} \Bigg\{ \mu^2 \sum_{p=1}^P \Bigg(  \left( \frac{2\delta^2}{\iota_2(1- \iota_2)} + \beta_{s,p}^2   + \frac{O(\mu)}{(1-\iota_2)^2}\right) \notag \\
	& \quad \times  \sum_{m\in\mathcal{N}_p} \frac{\iota_2(\mu^2 \sigma_{s,m}^2  + O(\mu^2)\xi^2 + O(\mu^3)\sigma_{q,m}^2+ \sigma_{g,m}^2)  }{1-\iota_2\lambda_p a_{pm}}  
	\notag \\
	& \quad + \frac{2\Vert \grad{w}J_p(w^o)\Vert^2}{\iota_2 (1-\iota_2)^2 }  + \frac{\sigma_{s,p}^2 }{1-\iota_2}+ \frac{O(\mu)\xi^2 + O(\mu^2)\sigma_{q,p}^2}{(1-\iota_2)^3}\Bigg) \notag \\
	&\quad + \frac{1}{(1-\iota_2)\iota_2^2}\sum_{p=1}^P \sigma_{g,p}^2 \Bigg\}
	 \notag \\
	& = 
	\frac{\iota_2^2}{P(1-\iota_2)} \sum_{p=1}^P \mu^2 \sigma_{s,p}^2 + \frac{1}{\iota_2^2}\sigma_{g,p}^2  + O(\mu)\sigma_{g,p}^2 + O(\mu^3).
 \end{align}

\section{Proof of Theorem \ref{thrm:centMSE}}\label{app:thrmCentMSE}

	First taking the conditional mean of the $\ell_2-$norm of the centroid error given the past models, splits the mean into three independent terms: the centralized recursion, the gradient noise and the added noise. Then, taking the expectation again, we get:
	\begin{align}
	&	\mathbb{E} \Vert \we_{c,i}\Vert^2 \notag \\
		 &= \mathbb{E} \bigg\Vert \we_{c,i-1} + \mu \frac{1}{P}\sum_{p=1}^P \grad{w}J_p(\w_{p,i-1}) + \mu \bm{q}_i \bigg\Vert^2 + \mu^2 \mathbb{E} \Vert \bm{s}_i\Vert^2\notag \\
		&\quad + \mathbb{E } \Vert \bm{g}_{c,i}\Vert^2 \notag \\
		& \stackrel{(a)}{\leq} \frac{1}{\alpha^2}\mathbb{E} \bigg\Vert \we_{c,i-1} + \mu \frac{1}{P}\sum_{p=1}^P \grad{w}J_p(\w_{c,i-1}) \bigg\Vert^2 + \frac{\mu^2}{1-\alpha} \mathbb{E}\Vert \bm{q}_i \Vert^2 \notag \\
		&\quad+ \frac{\delta^2\mu^2 }{\alpha(1-\alpha)P}\sum_{p=1}^P \mathbb{E} \Vert \w_{p,i-1} - \w_{c,i-1}\Vert^2  + \mu^2\mathbb{E} \Vert \bm{s}_i\Vert^2+ \mathbb{E} \Vert \bm{g}_{c,i}\Vert^2 \notag \\
		& \stackrel{(b)}{\leq} \left( \frac{1}{\alpha^2}(1-2\nu\mu +  \delta^2 \mu^2)+ \beta_s^2\mu^2 + \frac{O(\mu^3)}{1-\alpha} \right)\mathbb{E}\Vert \we_{c,i-1}\Vert^2 \notag \\
		&\quad +\mu^2 \sigma_s^2 + \mathbb{E}\Vert \bm{g}_{c,i}\Vert^2 + \left(\frac{\delta^2}{\alpha(1-\alpha)} +\frac{O(\mu^3)}{1-\alpha}+ \beta_{s,max}^2 \right)\frac{\mu^2}{P} \notag \\
		&\quad \times \sum_{p=1}^P \mathbb{E}\Vert \w_{p,i-1} - \w_{c,i-1}\Vert^2 + \frac{O(\mu^3)\xi^2 + O(\mu^4)\sigma_{q}^2}{1-\alpha}, 
	\end{align} 
where inequality $(a)$ follows from Jensen with constant $\alpha \in (0,1)$ and Lipshcitz, and (b) from applying Lemma \ref{lem:gradNoise}. Then, choosing $\alpha = \sqrt[4]{1-2\nu\mu + \delta^2 \mu^2} = 1 - O(\mu)$, the bound becomes:
\begin{align}
	\mathbb{E} \Vert \we_{c,i}\Vert^2 \leq &\lambda_c \mathbb{E}\Vert \we_{c,i-1}\Vert^2 + \mu^2 \sigma_s^2 + \mathbb{E} \Vert \bm{g}_{c,i}\Vert^2  + O(\mu^{2})\xi^2 \notag \\
	& + O(\mu^{3})\sigma_{q}^2 + \frac{O(\mu)}{P} \sum_{p=1}^P\mathbb{E}\Vert \w_{p,i-1}-\w_{c,i-1}\Vert^2.
\end{align}

Finally, using the result on the network disagreement, recusrively bounding the error, and taking the limit of $i$, we get the final result:
\begin{align}
	\limsup_{i \to  \infty} \mathbb{E} \Vert \we_{c,i}\Vert ^2 \leq & \frac{\mu^2 \sigma_s^2 + \mathbb{E}\Vert \bm{g}_{c}\Vert^2 + O(\mu^{2})\xi^2 + O(\mu^{3})\sigma_q^2}{1-\lambda_c} \notag\\
	&+ \sum_{p=1}^PO(1 ) \sigma_{g,p}^2 + O(\mu).
\end{align}

\section{Secondary Result: Bound on $\Vert \we_{p,i}\Vert$ }\label{app:thrmBdL2Err}
In the theorem below we bound the $\ell_2-$norm of the error, which will be used to prove the boundedness of the gradients at every step of the algorithm.

\begin{theorem}[\textbf{Bounded $\ell_2-$norm of individual errors}]\label{thrm:bdL2Err}
	Under Assumptions \ref{ass:adj}, \ref{ass:fct} and \ref{ass:mod}, the $\ell_2-$norm of the individual errors is bounded as follows for a sufficiently small step-size:
	\begin{align}
		\col{\Vert \we_{p,i}\Vert}_{p=1}^P \preceq \mathds{1}^\tran \mathds{1}\col{\Vert \we_{p,0}\Vert }_{p=1}^P + O(1)\mathds{1}. 
	\end{align}
\end{theorem}

\begin{proof}
	We can bound the squared $\ell_2-$norm of the gradient noise by using Jensen's inequality and the Lipschitz condition:
	\begin{align}
		\Vert \bm{s}_{p,i}\Vert^2 &\leq 6\delta^2 \Vert \we_{p,i-1}\Vert^2 + 6\Vert \grad{w} J_{p}(w^o)\Vert^2 
		\notag \\
		&\quad + \frac{6}{L}\sum_{k=1}^K \Vert \grad{w}Q_{p,k}(w^o;x_{p,k,\text{max}})\Vert^2,
	\end{align}
where $x_{p,k,\text{max}}$ is the data sample that has the largest gradient. Similarly, the individual gradient noise defined in equation (26) of \cite{rizk2020federated} can be bounded by:
\begin{align}
	\Vert \bm{q}_{p,k,i,e}\Vert^2 &\leq 6\delta^2 \Vert \we_{p,k,e-1}\Vert^2 + 3\Vert \grad{w}Q_{p,k}(w^o_{p,k};x_{p,k,\text{max}})\Vert^2,
\end{align}
where the error is now defined with respect to the local optimal models $w_{p,k}^o$. This bound implies the following result on the local model errors:
\begin{align}
	\Vert \we_{p,k,e}\Vert^2 &\stackrel{(a)}{\leq} \frac{1}{\alpha} (1-2\nu\mu + \delta^2\mu^2)\Vert  \we_{p,k,e-1}\Vert^2 + \frac{\mu^2}{1-\alpha} \Vert \bm{q}_{p,k,i,e}\Vert^2 
	\notag \\
	&\leq \lambda_{p,k}^e \Vert \we_{p,k,0}\Vert^2 
	\notag \\
	&\quad + \frac{3\mu^2}{1-\alpha} \frac{1-\lambda_{p,k}^e}{1-\lambda_{p,k}} \Vert \grad{w}Q_{p,k}(w^o_{p,k};x_{p,k,\text{max}})\Vert^2
	\notag \\
	&= \lambda_{p,k}^e \Vert \we_{p,i-1} + w_{p,k}^o - w^o\Vert^2
		\notag \\
	&\quad + \frac{3\mu^2}{1-\alpha} \frac{1-\lambda_{p,k}^e}{1-\lambda_{p,k}} \Vert \grad{w}Q_{p,k}(w^o_{p,k};x_{p,k,\text{max}})\Vert^2 
	\notag \\
	&\stackrel{(b)}{\leq} 2\lambda_{p,k}^e \Vert \we_{p,i-1}\Vert^2 + 2\lambda_{p,k}^e \xi^2 
		\notag 
\end{align} \begin{align}
	&\quad + \frac{3\mu^2}{1-\alpha} \frac{1-\lambda_{p,k}^e}{1-\lambda_{p,k}} \Vert \grad{w}Q_{p,k}(w^o_{p,k};x_{p,k,\text{max}})\Vert^2 ,
\end{align}
where (a) and (b) follow from Jensen's inequality with $\alpha = \sqrt{1-2\nu\mu + \delta^2 \mu^2}$
and $\lambda_{p,k} \eqdef \alpha + 6\delta^2\mu^2/(1-\alpha)$. Then using the above two inequalities, we can bound the incremental noise by adding and subtracting the gradient evaluated at the local optimal models and the global optimal model, then applying Jensen's and the Lipschitz condition:
\begin{align}
	\Vert \bm{q}_{p,i}\Vert^2 & \leq \frac{3\delta^2}{L}\sum_{k\in \Li{p}} \frac{1}{E_{p,k}}\sum_{e=1}^{E_{p,k}} \Vert \we_{p,k,e-1}\Vert^2 
	\notag \\
	&\quad + 3\delta^2 \xi^2 + 3\delta^2 \Vert \we_{p,i-1}\Vert^2
	\notag \\
	&\leq \frac{3\delta^2}{L} \sum_{k\in \Li{p}} \frac{1}{E_{p,k}} \sum_{e=1}^{E_{p,k}} 2\lambda_{p,k}^{e-1} \Vert \we_{p,i-1}\Vert^2 + 2\lambda_{p,k}^{e-1} \delta^2 \xi^2 
	\notag \\
	&\quad + \frac{3\mu^2}{1-\alpha} \frac{1-\lambda_{p,k}^{e-1}}{1-\lambda_{p,k}} \Vert \grad{w}Q_{p,k}(w_{p,k}^o;x_{p,k,\text{max}})\Vert^2 
	\notag \\
	&\quad + 3\delta^2 \xi^2 + 3\delta^2\Vert \we_{p,i-1}\Vert^2
	\notag \\
	&\leq 9\delta^2 \Vert \we_{p,i-1}\Vert^2 + 9\delta^2 \xi^2 
	\notag \\
	& \quad+ \frac{9\delta^2}{L}\frac{\mu^2}{1-\alpha}\sum_{k=1}^K \frac{ \Vert \grad{w}Q_{p,k}(w_{p,k}^o;x_{p,k,\text{max}})\Vert^2}{1-\lambda_{p,k}} .
\end{align}
Using Jensen's inequality again with the same $\alpha$ and where $\lambda \eqdef \alpha + 30\delta^2\mu^2/(1-\alpha)$, we have:
\begin{align}
	\Vert \we_{p,i}\Vert^2 &\leq  \sum_{m\in \mathcal{N}_p}a_{mp} \Bigg\{ \frac{1}{\alpha} \Vert \we_{m,i-1} + \mu \grad{w}J_m(\w_{m,i-1})\Vert^2 
	\notag \\
	&\quad + \frac{2\mu^2}{1-\alpha}\Vert \bm{s}_{m,i}\Vert^2 
	+\frac{2\mu^2}{1-\alpha}\Vert \bm{q}_{m,i}\Vert^2
	\Bigg\}
	\notag \\
	&\leq \sum_{m\in \mathcal{N}_p} a_{mp}\Bigg\{ \lambda\Vert \we_{m,i-1}\Vert^2 + \frac{18\delta^2\xi^2\mu^2}{1-\alpha}
	\notag \\
	&\quad + \frac{12\mu^2}{L(1-\alpha)}\sum_{k=1}^K \Vert \grad{w}Q_{p,k}(w^o;x_{p,k,\text{max}})\Vert^2 
	\notag \\
	&\quad + \frac{18\delta^2 \mu^4}{L(1-\alpha)^2}\sum_{k=1}^K  \frac{ \Vert \grad{w}Q_{p,k}(w_{p,k}^o;x_{p,k,\text{max}})\Vert^2}{1-\lambda_{p,k}} 
	\Bigg\}.
\end{align}

The bound on the column vector of the errors becomes:
\begin{align}
	\col{\Vert \we_{p,i}\Vert^2}_{p=1}^P &\preceq \lambda \mathcal{A} \col{\Vert \we_{p,i-1}\Vert^2}_{p=1}^P +O(\mu)\mathds{1}
	\notag \\
	&\preceq \mathds{1}^\tran \mathds{1}\col{\Vert \we_{p,0}\Vert^2}_{p=1}^P + O(1)\mathds{1},
\end{align}
and finally:
\begin{align}
	\Vert \we_{p,i}\Vert \leq \sum_{m=1}^P \Vert \we_{m,0}\Vert + O(1).
\end{align}
\end{proof}

\vspace*{-1cm}
\section{Proof of Lemma \ref{lem:bdGrad}}\label{app:lemBdGrad}
	Using Theorem \ref{thrm:bdL2Err} in Appendix \ref{app:thrmBdL2Err}, we can bound the local gradients during each iteration as follows:
		\begin{align}
		&\Vert\grad{w}Q_{p,k}(\w_{p,i};\bm{x}_{p,k,b}) \Vert 
		\notag \\
		&= \Vert \grad{w}Q_{p,k}(\w_{p,i};\bm{x}_{p,k,b})) - \grad{w}Q_{p,k}(w^o;\bm{x}_{p,k,b})) 
		\notag \\
		&\quad + \grad{w}Q_{p,k}(w^o;\bm{x}_{p,k,b})) \Vert 
		\notag 
\end{align} \begin{align}
		&\leq \delta \Vert \we_{p,i}\Vert + \Vert \grad{w}Q_{p,k}(w^o;x_{p,k,\text{max}}) \Vert, 
	\end{align}
	where both terms on the right hand side are bounded.

\section{Proof of Theorem \ref{thrm:priv}}\label{app:thrmPriv}

	To evaluate the probability distribution in definition \ref{def:DP}, we note that the randomness of the models $\bm{\psi}_{p,j}$ arises from the subsampling of the data for the calculation of the stochastic gradient at each iteration. Thus, given the subsampled dataset, the models are now deterministic and since the added noises $\bm{g}_{pm,j}$ are Laplacian random variables, the distribution of the added noise over the neighbourhood of agent $p$ and over the iterations is given by:
	\begin{align}
		&\mathbb{P}\left( \left\{   \left\{ \bm{\psi}_{p,j} + \bm{g}_{pm,j}   \right\}_{m\in \mathcal{N}_p \setminus \{p\} } \right\} _{j=0}^i \bigg | \mathcal{F}_{i} \right) \notag \\
		&= \prod_{j=0}^i \frac{1}{\sqrt{2}\sigma_g} \exp \bigg( -\frac{\sqrt{2}}{\sigma_g} | \bm{\psi}_{p,j} + \bm{g}_{p,j} |\bigg) \notag \\
		&= \frac{1}{\sqrt{2}\sigma_g} \exp \bigg( -\frac{\sqrt{2}}{\sigma_g}  \sum_{j=0}^i | \bm{\psi}_{p,j} + \bm{g}_{p,j}|  \bigg),
	\end{align}
	and the ratio in Definition \ref{def:DP} is equal to:
	\begin{align}
		&\exp \bigg(-\frac{\sqrt{2}}{\sigma_g} \sum_{j=0}^i |\bm{\psi}_{p,j} + \bm{g}_{p,j}| -  |\bm{\psi}'_{p,j} + \bm{g}_{p,j}| \bigg) \notag \\
		&\leq \exp \bigg(  \frac{\sqrt{2}}{\sigma_g} \sum_{j=0}^i |\bm{\psi}_{p,j} - \bm{\psi}'_{p,j}  | \bigg) \notag \\
		&\leq \exp \bigg(  \frac{\sqrt{2}}{\sigma_g} \sum_{j=0}^i2\mu B j \bigg) \notag \\
		&= \exp \bigg(\frac{\sqrt{2}}{\sigma_g } \mu B (i+1)i \bigg),
	\end{align} 
where the inequalities follow from the triangle inequality and the bound on the sensitivity of the algorithm.

\end{document}